%% file: main.tex
\documentclass{article}

\input{src/preamble}

\title{A Block-wise, Asynchronous and Distributed ADMM Algorithm \\for General Form Consensus Optimization}

\begin{document}

\author{
Rui Zhu$^1$, Di Niu$^1$, Zongpeng Li$^2$\\
${}^1$ Department of Electrical and Computer Engineering, University of Alberta\\
${}^2$ School of Computer Science, Wuhan University\\
 \{\texttt{rzhu3, dniu}\}\texttt{@ualberta.ca}, \texttt{zongpeng@whu.edu.cn} 
}

\maketitle


\input{src/abstract}
\input{src/introduction}

\input{src/formulation}

\input{src/algorithm}
\input{src/theory}
\input{src/simu}
\input{src/conclude}

\bibliographystyle{named}
\bibliography{src/main}

\include{src/appendix}

\end{document}

%% file: src/preamble.tex
\usepackage{ijcai18}
\usepackage{times}
\usepackage{xcolor}
\usepackage{soul}
\usepackage[utf8]{inputenc}
\usepackage[small]{caption}
\usepackage{hyperref}

\usepackage{amsmath}
\usepackage{amsthm}
\usepackage{amssymb}
\usepackage{cases}
\usepackage{proof}
\usepackage{textcomp}

\usepackage{algorithm}
\usepackage{algpseudocode}
\makeatletter
\def\BState{\State\hskip-\ALG@thistlm}
\makeatother

\usepackage{color}
\newcommand{\red}[1]{{#1}}

\usepackage{enumerate}

\usepackage{bm, bbm}

\usepackage{subfigure}

\usepackage{booktabs}

\usepackage{multicol}

\usepackage{pgf,tikz}
\usetikzlibrary{shapes,automata,snakes,backgrounds,arrows}
\usetikzlibrary{mindmap}

\def\squareforqed{\hbox{\rlap{$\sqcap$}$\sqcup$}}
\def\qed{\ifmmode\squareforqed\else{\unskip\nobreak\hfil
\penalty50\hskip1em\null\nobreak\hfil\squareforqed
\parfillskip=0pt\finalhyphendemerits=0\endgraf}\fi}

\newtheorem{thm}{Theorem}

\newtheorem{lem}{Lemma}
\newtheorem{cor}[thm]{Corollary}

\newtheorem{asmp}{Assumption}

\newcommand{\real}[1]{\mathbb{R}^{{#1}}}
\newcommand{\innerprod}[1]{\langle {#1} \rangle}
\newcommand{\norm}[1]{\lVert {#1} \rVert}

\newcommand{\prox}[2][{}]{\mathrm{prox} {}^{#1}_{#2}}

%% file: src/abstract.tex

\begin{abstract}
Many machine learning models, including those with non-smooth regularizers, can be formulated as consensus optimization problems, which can be solved by the alternating direction method of multipliers (ADMM). Many recent efforts have been made to develop asynchronous distributed ADMM to handle large amounts of training data.
However, all existing asynchronous distributed ADMM methods are based on full model updates and  require locking all global model parameters to handle concurrency, which essentially serializes the updates from different workers. In this paper, we present a novel block-wise, asynchronous and distributed ADMM algorithm, which allows different blocks of model parameters to be updated in parallel. The lock-free block-wise algorithm may greatly speedup sparse optimization problems, a common scenario in reality, in which most model updates only modify a subset of all decision variables. We theoretically prove the convergence of our proposed algorithm to stationary points for non-convex general form consensus problems with possibly non-smooth regularizers.
We implement the proposed ADMM algorithm on the Parameter Server framework and demonstrate its convergence and near-linear speedup performance as the number of workers increases.
\end{abstract}

%% file: src/introduction.tex

\section{Introduction}
\label{sec:intro}



The need to scale up machine learning in the presence of sheer volume of data 
has spurred recent interest in developing efficient distributed optimization algorithms. 
Distributed machine learning jobs often involve solving a non-convex, decomposable, and regularized optimization problem of the following form:
\begin{equation}
\begin{split}
	\mathop{\min}_{\mathbf{x}} &\quad \sum_{i=1}^N f_i(x_1,\ldots,x_{M}) + \sum_{j=1}^M h_j(x_j),  \\
	\mathrm{s.t.}&\quad x_j \in \mathcal{X}_j, j=1,\ldots,M
\end{split}
\label{eq:original}
\end{equation}
where each $f_i:\mathcal{X} \to \real{}$ is a smooth but possibly \emph{non-convex} function, fitting the model $\mathbf{x} :=(x_1,\ldots,x_M)$ to local training data available on node $i$; each $\mathcal{X}_j$ is a closed, convex, and compact set; and the regularizer $h(\mathbf{x}):=\sum_{j=1}^M h_j(x_j)$ is a separable, convex but possibly \emph{non-smooth} regularization term to prevent overfitting. 
Example problems of this type can be found in deep learning with regularization \cite{dean2012large,chen2015mxnet}, robust matrix completion \cite{recht2011hogwild}, LASSO \cite{tibshirani2005sparsity}, sparse logistic regression \cite{liu2009large}, and sparse support vector machine (SVM) \cite{friedman2001elements}.


To date, a number of efficient asynchronous and distributed stochastic gradient descent (SGD) algorithms, e.g., \cite{recht2011hogwild,lian2015asynchronous,li2014scaling}, have been proposed, in which each worker node asynchronously updates its local model or gradients based on its local dataset, and sends them to the server(s) for model updates or aggregation. Yet, SGD is not particularly suitable for solving optimization problems with non-smooth objectives or with constraints, which are prevalent in practical machine learning adopting regularization, e.g., \cite{liu2009large}. Distributed (synchronous) ADMM \cite{boyd2011distributed,zhang2014asynchronous,chang2016asynchronous1,chang2016asynchronous2,hong2017distributed,wei20131,mota2013d,taylor2016training} has been widely studied as an alternative method, which avoids the common pitfalls of SGD for highly non-convex problems, such as saturation effects, poor conditioning, and saddle points \cite{taylor2016training}.
The original idea on distributed ADMM can be found in \cite{boyd2011distributed}, which is essentially a synchronous algorithm. 
In this work, we focus on studying the asynchronous distributed alternating direction method of multipliers (ADMM) for non-convex non-smooth optimization.


Asynchronous distributed ADMM has been actively discussed in recent literature. 
Zhang and Kwok \shortcite{zhang2014asynchronous} consider an asynchronous ADMM assuming bounded delay, which enables each worker node to update a local copy of the model parameters asynchronously without waiting for other workers to complete their work, while a single server is responsible for driving the local copies of model parameters to approach the global \emph{consensus variables}. They provide proof of convergence for convex objective functions only. Wei and Ozdaglar \shortcite{wei20131} assume that communication links between nodes can fail randomly, and propose an ADMM scheme that converges almost surely to a saddle point. Chang \emph{et al.} \shortcite{chang2016asynchronous1,chang2016asynchronous2} propose an asynchronous ADMM algorithm with analysis for non-convex objective functions. However, their work requires each worker to solve a subproblem \emph{exactly}, which is often costly in practice. Hong \shortcite{hong2017distributed} proposes another asynchronous ADMM algorithm, where each worker only computes the gradients based on local data, while all model parameter updates happen at a single server, a possible bottleneck in large clusters.

To our knowledge, all existing work on asynchronous distributed ADMM requires locking global consensus variables at the (single) server for each model update; although asynchrony is allowed among workers, i.e., workers are allowed to be at different iterations of model updating. Such atomic or memory-locking operations essentially serialize model updates contributed by different workers, which may seriously limit the algorithm scalability. In many practical problems, not all workers need to access all model parameters. For example, in recommender systems, a local dataset of user-item interactions is only associated with a specific set of users (and items), and therefore does not need to access the latent variables of other users (or items). In text categorization, each document usually consists of a subset of words or terms in corpus, and each worker only needs to deal with the words in its own local corpus. 
 
\begin{figure}
  \centering
  \includegraphics[width=3in]{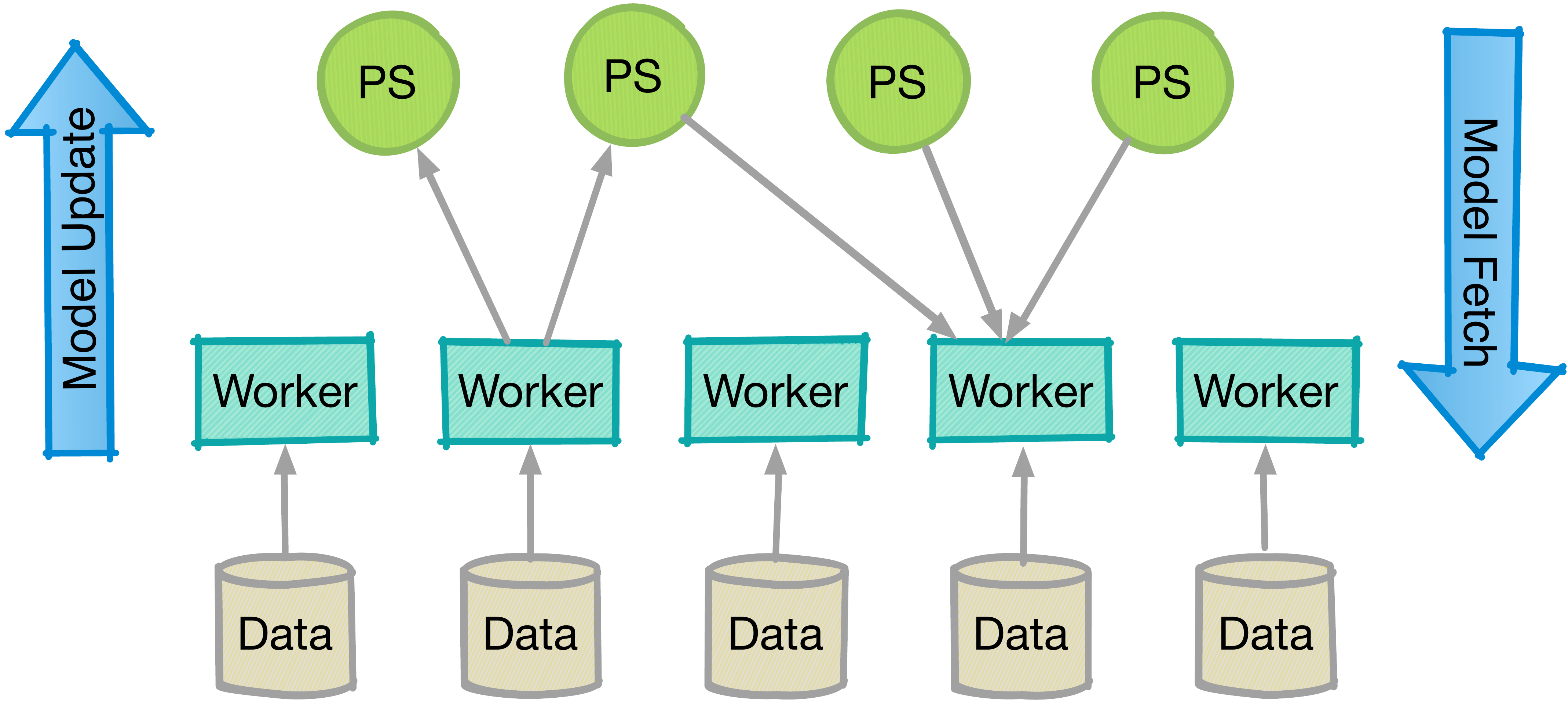}
  \caption{Data flow of Parameter Servers. ``PS'' represents a parameter server task and ``Worker'' represents a worker task.}
  \label{fig:ps}
\end{figure}

A distributed machine learning algorithm is illustrated in Fig.~\ref{fig:ps}. There are multiple server nodes, each known as a ``PS'' and stores a subset (block) of model parameters (consensus variables) $\mathbf{z}$. There are also multiple worker nodes; each worker owns a local dataset, and has a loss function $f_i$ depending on one or several blocks of model parameters, but not necessarily all of them. If there is only one server node, the architecture in Fig.~\ref{fig:ps} degenerates to a ``star'' topology with a single master, which has been adopted by Spark \cite{zaharia2010spark}. With multiple servers, the system is also called a \emph{Parameter Server} architecture \cite{dean2012large,li2014scaling} and has been adopted by many large scale machine learning systems, including TensorFlow \cite{abadi2016tensorflow} and MXNet \cite{chen2015mxnet}.

It is worth noting that enabling \emph{block-wise} updates in ADMM is critical for training large models, such as sparse logistic regression, robust matrix completion, etc., since not all worker nodes will need to work on all model parameters --- each worker only needs to work on the blocks of parameters pertaining to its local dataset.
For these reasons, block-wise updates have been extensively studied for a number of gradient type of distributed optimization algorithms, including SGD \cite{lian2015asynchronous}, proximal gradient descent \cite{li2014communication}, block or stochastic coordinate descent (BCD or SCD) \cite{liu2015asynchronous}, as well as for a recently proposed block successive upper bound minimization method (BSUM) \cite{hong2016unified}.

In this work, we propose the first \emph{block-wise} asynchronous distributed ADMM algorithm that can increase efficiency over existing single-server ADMM algorithms, by better exploiting the parallelization opportunity in model parameter updates. Specifically, we introduce the \emph{general form consensus optimization} problem \cite{boyd2011distributed}, and solve it in a \emph{block-wise} asynchronous fashion, thus making ADMM amenable for implementation on  Parameter Server, with multiple servers hosting model parameters. 
In our algorithm, each worker only needs to work on one or multiple blocks of parameters that are relevant to its local data, while different blocks of model parameters can be updated in parallel asynchronously subject to a bounded delay. 
Since this scheme does not require locking all the decision variables together, it belongs to the set of \emph{lock-free} optimization algorithms (e.g., HOGWILD! \cite{recht2011hogwild} as a lock-free version of SGD) in the literature.
Our scheme is also useful on shared memory systems, such as on a single machine with multi-cores or multiple GPUs, where enforcing atomicity on all the consensus variables is inefficient.


Theoretically, we prove that, for general \emph{non-convex} objective functions, our scheme can converge to stationary points. 
Experimental results on a cluster of 36 CPU cores have demonstrated the convergence and near-linear speedup of the proposed ADMM algorithm, for training sparse logistic regression models based on a large real-world dataset.

%% file: src/formulation.tex
\section{Preliminaries}
\label{sec:prelim}

\subsection{Consensus Optimization and ADMM}
\label{sec:basicADMM}
The minimization in \eqref{eq:original} can be reformulated into a \emph{global variable consensus optimization} problem \cite{boyd2011distributed}:
\begin{subequations}
\begin{align}
\mathop{\min}_{\mathbf{z}, \{\mathbf{x}_i\} \in \mathcal{X}} &\quad \sum_{i=1}^{N} f_i(\mathbf{x}_i) + h(\mathbf{z}), \label{eq:consensus-1}\\
\mathrm{s.t.} &\quad \mathbf{x}_i = \mathbf{z}, \quad \forall i = 1,\ldots,N, \label{eq:consensus-2}
\end{align}
\label{eq:consensus}
\end{subequations}
where $\mathbf{z}$ is often called the \emph{global consensus variable}, traditionally stored on a master node, and $\mathbf{x}_i$ is its local copy updated and stored on one of $N$ worker nodes. The function $h$ is decomposable. It has been shown \cite{boyd2011distributed} that such a problem can be efficiently solved using distributed (synchronous) ADMM. In particular, let $\mathbf{y}_i$ denote the Lagrange dual variable associated with each constraint in \eqref{eq:consensus-2} and define the Augmented Lagrangian as
\begin{equation}
\begin{split}
L(\mathbf{X}, \mathbf{Y}, \mathbf{z})
&= \sum_{i=1}^{N} f_i(\mathbf{x}_i) +  h(\mathbf{z}) 
 + \sum_{i=1}^{N} \innerprod{\mathbf{y}_{i}, \mathbf{x}_{i} - \mathbf{z}} \\
 &\quad + \sum_{i=1}^{N}  \frac{\rho_{i}}{2} \norm{\mathbf{x}_{i} - \mathbf{z} }^2,
\end{split}
\label{eq:lagrangian_1}
\end{equation}
where $\mathbf{X}:=(\mathbf{x}_1, \ldots, \mathbf{x}_N)$ represents a juxtaposed matrix of all $\mathbf x_i$ , and $\mathbf{Y}$ represents the juxtaposed matrix of all $\mathbf{y}_i$. We have, for (synchronized) rounds $t=0,1,\ldots$, the following variable updating equations:
{\small
\begin{align*}
 	\mathbf{x}_i^{t+1} &= \mathop{\arg\min}_{\mathbf{x}_i \in \mathcal{X}} f_i(\mathbf{x}_i) 
 +  \innerprod{\mathbf{y}_{i}^t, \mathbf{x}_{i} - \mathbf{z}^{t}} + \frac{\rho_{i}}{2} \norm{\mathbf{x}_{i} - \mathbf{z}^{t} }^2,\\
 	\mathbf{y}_i^{t+1} &= \mathbf{y}_i^{t} + \rho_i (\mathbf{x}_i^{t+1} - \mathbf{z}^{t}), \\
	\mathbf{z}^{t+1} &= \mathop{\arg\min}_{\mathbf{z} \in \mathcal{X}} h(\mathbf{z}) 
 + \sum_{i=1}^{N} \innerprod{\mathbf{y}_{i}^{t}, \mathbf{x}_{i}^t - \mathbf{z}^t} + \sum_{i=1}^{N} \frac{\rho_{i}}{2} \norm{\mathbf{x}^t_{i} - \mathbf{z}^t }^2.
\end{align*}
}

\subsection{General Form Consensus Optimization}

Many machine learning problems involve highly sparse models, in the sense that each local dataset on a worker is only associated with a few model parameters, i.e., each $f_i$ only depends on a subset of the elements in $\mathbf{x}$.
The global consensus optimization problem in \eqref{eq:consensus}, however, ignores such sparsity, since in each round each worker $i$ must push the entire vectors $\mathbf{x}_i$ and $\mathbf{y}_i$ to the master node to update $\mathbf z$. 
\red{In fact, this is the setting of all recent work on asynchronous distributed ADMM, e.g., \cite{zhang2014asynchronous}}.
In this case, when multiple workers attempt to update the global census variable $\mathbf z$ at the same time, $\mathbf z$ must be locked to ensure atomic updates, which leads to diminishing efficiency as the number of workers $N$ increases.

To better exploit model sparsity in practice for further parallelization opportunities between workers, we consider the \emph{general form consensus optimization} problem \cite{boyd2011distributed}. Specifically, with $N$ worker nodes and $M$ server nodes, the vectors $\mathbf{x}_i$, $\mathbf{y}_i$ and $\mathbf{z}$ can all be decomposed into $M$ blocks. Let $z_j$ denote the $j$-th block of the global consensus variable $\mathbf{z}$, located on server $j$, for $j=1,\ldots,M$.
Similarly, let $x_{i,j}$ ($y_{i,j}$) denote the corresponding $j$-th block of the local variable $\mathbf{x}_i$ ($\mathbf{y}_i$) on worker $i$. Let $\mathcal E$ be all the $(i,j)$ pairs such that $f_i$ depends on the block $x_{i,j}$ (and correspondingly depends on $z_j$). Furthermore, let $\mathcal N(j) = \{i|(i,j)\in \mathcal E\}$ denote the set of all the neighboring workers of server $j$. Similarly, let $\mathcal N(i) = \{j|(i,j)\in \mathcal E\}$.

Then, the \emph{general form consensus problem} \cite{boyd2011distributed} is described as follows:
\begin{equation}
\begin{split}
	\mathop{\min}_{z_j, \{x_{i,j}\}}  & \quad \sum_{i=1}^{N} f_i(\{x_{i,j}\}_{j=1}^M) + h(\mathbf{z}), \\
	\mathrm{s.t.} &\quad x_{i,j} = z_j, \quad \forall (i,j) \in \mathcal{E}, \\
	&\quad x_{i,j}, z_j \in \mathcal{X}_j.
\end{split}
\label{eq:general-consensus}
\end{equation}
In fact, in $f_i(\{x_{i,j}\}_{j=1}^M)$, a block $x_{i,j}$ will only be relevant if $(i,j)\in \mathcal E$, and will be a dummy variable otherwise, whose value does not matter. Yet, since the sparse dependencies of $f_i$ on the blocks $j$ can be captured through the specific form of $f_i$, here we have included all $M$ blocks in each $f_i$'s arguments just to simplify the notation.

The structure of problem \eqref{eq:general-consensus} can effectively capture the sparsity inherent to many practical machine learning problems. Since each $f_i$ only depends on a few blocks, the formulation in \eqref{eq:general-consensus} essentially reduces the number of decision variables---it does not matter what value $x_{i,j}$ will take for any $(i,j)\notin\mathcal E$. \red{For example, when training a topic model for documents, the feature of each document is represented as a bag of words, and hence only a subset of all words in the vocabulary will be active in each document's feature. In this case, the constraint $x_{i,j}=z_j$ only accounts for those words $j$ that appear in the document $i$, and therefore only those words $j$ that appeared in document $i$ should be optimized.} 
Like \eqref{eq:lagrangian_1}, we also define the Augmented Lagrangian $L(\mathbf{X}, \mathbf{Y}, \mathbf{z})$\footnote{To simplify notations, we still use $\mathbf{X}$ and $\mathbf{Y}$ as previously defined, but entries $(i,j) \notin \mathcal{E}$ are not taken into account.} as follows:
{\small
\begin{align*}
L(\mathbf{X}, \mathbf{Y}, \mathbf{z})
&= \sum_{i=1}^{N} f_i(\{x_{i,j}\}_{j=1}^M) +  h(\mathbf{z})
 + \sum_{(i,j) \in \mathcal{E}} \innerprod{{y}_{i,j}, {x}_{i,j} - {z}_j} \nonumber \\
 &\quad + \sum_{(i,j) \in \mathcal{E}}  \frac{\rho_{i}}{2} \norm{{x}_{i,j} - {z}_j }^2.
\label{eq:lagrangian_2}
\end{align*}
}

The formulation in \eqref{eq:general-consensus} perfectly aligns with the latest \emph{Parameter Server} architecture as shown in Fig.~\ref{fig:ps}. Here we can let each server node maintain one model block $z_j$, such that worker $i$ updates $z_j$ if and only if $(i, j)\in \mathcal{E}$.
Since all three vectors $\mathbf x_i$, $\mathbf y_i$ and $\mathbf z$ in \eqref{eq:general-consensus} are decomposable into blocks, to achieve a higher efficiency, we will investigate block-wise algorithms which not only enable different workers to send their updates asynchronously to the server (like prior work on asynchronous ADMM does), but also enable different model blocks $\mathbf z_j$ to be updated in parallel and asynchronously on different servers, removing the locking or atomicity assumption required for updating the entire $\mathbf z$.
  



%% file: src/algorithm.tex

\section{A Block-wise, Asynchronous, and Distributed ADMM Algorithm}
\label{sec:algorithm}

In this section, we present our proposed \emph{block-wise}, \emph{asynchronous} and \emph{distributed} ADMM algorithm (a.k.a, AsyBADMM) for the general consensus problem. For ease of presentation, we first describe a synchronous version motivated by the basic distributed ADMM for \emph{non-convex} optimization problems as a starting point. 

\subsection{Block-wise Synchronous ADMM}
The update rules presented in Sec.~\ref{sec:basicADMM} represent the basic synchronous distributed ADMM approach \cite{boyd2011distributed}.
To solve the general form consensus problem, our block-wise version extends such a synchronous algorithm mainly by 1) approximating the update rule of $\mathbf x_i$ with a simpler expression under non-convex objective functions, and 2) converting the all-vector updates of variables into block-wise updates only for $(i,j)\in \mathcal E$. 
 
Generally speaking, in each synchronized epoch $t$, each worker node $i$ updates all blocks of its local primal variables $x_{i,j}$ and dual variables $y_{i,j}$ for $j\in\mathcal N(i)$, and pushes these updates to the corresponding servers.
Each server $j$, when it has received $x_{i,j}$ and $y_{i,j}$ from all $i\in\mathcal N(j)$, will update $z_j$ accordingly, by aggregating these received blocks.

Specifically, at epoch $t$, the basic synchronous distributed ADMM will do the following update for $\mathbf x_i$:
{\small
\begin{equation*}
\begin{split}
	\mathbf{x}_{i}^{t+1} &= \mathop{\arg\min}_{\mathbf{x}_i} f_i(\mathbf{x}_i)+ \sum_{j\in\mathcal{N}(i)} \innerprod{y_{i,j}^{t}, x_{i,j} - {z_j}^{t}} \\
	&\quad + \sum_{j\in\mathcal{N}(i)} \frac{\rho_i}{2} \norm{x_{i,j} - {z}_j^{t} }^2.
\end{split}
\end{equation*}	
}
However, this subproblem is hard, especially when $f_i$ is non-convex. To handle non-convex objectives, we adopt an alternative solution \cite{hong2017distributed,hong2016convergence} to this subproblem through the following first-order approximation of $f_i$ at $\mathbf z^t$:
{\small
\begin{align}
\mathbf{x}_i^{t+1} &\approx \mathop{\arg\min}_{\mathbf{x}_i} f_i({\mathbf{z}}^{t}) + \innerprod{\nabla f_i(\mathbf{z}^{t}), \mathbf{x}_i - \mathbf{z}^{t}} \nonumber \\
 &\quad + \sum_{j\in\mathcal{N}(i)}\left(\innerprod{y_{i,j}^{t}, x_{i,j} - {z}_j^{t}} + \frac{\rho_i}{2} \norm{x_{i,j} - {z}_j^{t} }^2 \right) \nonumber \\
 &= {\mathbf{z}}^{t} - \frac{\nabla f_i({\mathbf{z}}^{t})+\mathbf{y}_{i}^t}{\rho_{i}}, \label{eq:sub_x}
\end{align}
}
where \eqref{eq:sub_x} can be readily obtained by setting the partial derivative w.r.t. $\mathbf{x}_i$ to zero.  

The above full-vector update on $\mathbf x_i$ is equivalent to the following block-wise updates on each block $x_{i,j}$ by worker $i$:
\begin{equation}
	x_{i,j}^{t+1} = {z}_{j}^{t} - \frac{\nabla_j f_i({\mathbf{z}}^{t})+y_{i,j}^t}{\rho_i}, \label{eq:sync_update_x} 
\end{equation}
where $\nabla_j f_i(\mathbf{z}^{t+1})$ is the partial derivative of $f_i$ w.r.t. $z_j$. Furthermore, the dual variable blocks $y_{i,j}$ can also be updated in a block-wise fashion as follows:
\begin{equation}
	y_{i,j}^{t+1} = y_{i,j}^t + \rho_i(x_{i,j}^{t+1} - {z}_{j}^{t}).
	\label{eq:sync_update_y}
\end{equation}
Note that in fact, each $f_i$ only depends on a part of ${\mathbf{z}}^{t}$ and thus each worker $i$ only needs to pull the relevant blocks $z_j$ for $j\in\mathcal N(i)$. Again, we put the full vector $\mathbf z$ in $f_i(\cdot)$ just to simplify notation.

On the server side, server $j$ will update $z_j$ based on the newly updated $x^{t+1}_{i,j}$, $y^{t+1}_{i,j}$ received from all workers $i$ such that $i \in \mathcal{N}(j)$. Again, the $\mathbf z$ update in the basic synchronous distributed ADMM can be rewritten into the following block-wise format (with a regularization term introduced):  
\begin{align}
z_j^{t+1} &= \mathop{\arg\min}_{z_j \in \mathcal{X}_j} h_j(z_j) + \frac{\gamma}{2} \norm{z_j - z_j^t}^2
 \nonumber \\
&\quad + \sum_{i \in \mathcal{N}(j)} \left( \innerprod{y_{i,j}^{t+1}, x_{i,j}^{t+1} - z_j} + \frac{\rho_i}{2} \norm{x_{i,j}^{t+1} - z_j }^2 \right)   \nonumber \\
&= \prox[\mu]{h} \left( 
\frac{\gamma z_j^t + \sum_{i \in \mathcal{N}(j)}w_{i,j}^t }{\gamma+\sum_{i \in \mathcal{N}(j)} \rho_i}\right), \label{eq:sync_update_z}
\end{align}
where $w_{i,j}^{t+1}$ is defined as
\begin{equation}
	w_{i,j}^{t+1}:=\rho_i x_{i, j}^{t+1} + y_{i,j}^{t+1},
	\label{eq:push_w}
\end{equation}
and the proximal operator is defined as
\begin{equation}
	\prox[\mu]{h}(x) := \mathop{\arg\min}_{u \in \mathcal{X}_j} h(u) + \frac{\mu}{2} \norm{x - u}^2.
\end{equation}
Furthermore, the regularization term $\frac{\gamma}{2} \norm{z_j - z_j^t}^2$ is introduced to stabilize the results, which will be helpful in the asynchronous case.

In the update of $z_j$, the constant $\mu$ of the proximal operator is given by $\sum_{i\in \mathcal{N}(j)} \rho_i$. Now it is clear that it is sufficient for worker $i$ to send $w_{i,j}^t$ to server $j$ in epoch $t$.

\subsection{Block-wise Asynchronous ADMM}

We now take one step further to present a \emph{block-wise} \emph{asynchronous} \emph{distributed} ADMM algorithm, which is our main contribution in this paper.
In the asynchronous algorithm, each worker $i$ will use a local epoch $t$ to keep track of how many times  $\mathbf x_i$ has been updated, although different workers may be in different epochs, due to random delays in computation and communication.

Let us first focus on a particular worker $i$. While worker $i$ is in epoch $t$, there is no guarantee for worker $i$ to download $\mathbf z^t$---different blocks $z_j$ in $\mathbf z$ may have been updated for different numbers of times, for which worker $i$ has no idea. Therefore, we use $\tilde{z}_j^t$ to denote the \emph{latest copy} of $z_j$ on server $j$ while worker $i$ is in epoch $t$ and ${\tilde{\mathbf{z}}}^{t} = (\tilde{z}_1^t,\ldots, \tilde{z}_M^t)$. Then, the original synchronous updating equations \eqref{eq:sync_update_x} and \eqref{eq:sync_update_y} for $x_{i,j}$ and $y_{i,j}$, respectively, are simply replaced by
\begin{align}
	x_{i,j}^{t+1} &= \tilde{z}_{j}^{t} - \frac{\nabla_j f_i({\tilde{\mathbf{z}}}^{t})+y_{i,j}^t}{\rho_i}, \label{eq:update_x} \\
	y_{i,j}^{t+1} &= y_{i,j}^t + \rho_i(x_{i,j}^{t+1} - \tilde{z}_{j}^{t}). \label{eq:update_y}
\end{align}

Now let us focus on the server side. Since in the asynchronous case, the variables $w^{t+1}_{i,j}$ for different workers $i$ do not generally arrive at the server $j$ at the same time. In this case, we will update $z^{t+1}_j$ \emph{incrementally} as soon as a $w^{t}_{i,j}$ is received from some worker $i$ until all $w^{t}_{i,j}$ are received for all $i\in\mathcal N(j)$, at which point the update for $z^{t+1}_j$ is \emph{fully finished}. We use $\tilde z^{t+1}_j$ to denote the working (dirty) copy of $z^{t+1}_j$ for which the update may not be fully finished by all workers yet. Then, the update of $\tilde z^{t+1}_j$ is given by
\begin{align}
\tilde z_j^{t+1} 
&= \prox[\mu]{h} \left( 
\frac{\gamma \tilde z_j^t + \sum_{i \in \mathcal{N}(j)}\tilde w_{i,j}}{\gamma+\sum_{i \in \mathcal{N}(j)} \rho_i}\right), \label{eq:update_z}
\end{align}
where $\tilde w_{i,j} = w_{i,j}^t$ if $w_{i,j}^t$ is received from worker $i$ and triggering the above update; and for all other $i$, $\tilde w_{i,j}$ is the latest version of $w_{i,j}$ that server $j$ holds for worker $i$. The regularization coefficient $\gamma>0$ helps to stabilize convergence in the asynchronous execution with random delays.





\begin{algorithm}[htbp]
\caption{AsyBADMM: Block-wise Asynchronous ADMM}
\label{alg:Asybadmm}
\underline{Each \textbf{worker} $i$ asynchronously performs:}
\begin{algorithmic}[1]
\State pull $\mathbf{z}^0$ to initialize $\mathbf{x}^0=\mathbf{z}^0$
\State initialize $\mathbf{y}^0$ as the zero vector.
\For {$t=0$ \textbf{to} $T-1$}
	\State select an index $j_t\in\mathcal{N}(i)$ uniformly at random
	\State compute gradient $\nabla_{j_t} f(\mathbf{\tilde z}^{t})$
	\State update $x_{i, j_t}^{t+1}$ and $y_{i, j_t}^{t+1}$ by \eqref{eq:update_x} and \eqref{eq:update_y}
	\State push $w_{i,j_t}^{t+1}$ as defined by \eqref{eq:push_w} to server $j_t$
	\State pull the current models $\tilde{\mathbf{z}}^{t+1}$ from servers
\EndFor
\end{algorithmic}
\underline{Each \textbf{server} $j$ asynchronously performs:}
\begin{algorithmic}[1]
	\State initialize $\tilde z^0_{j}$ and $\tilde w_{i,j}$ for all $i \in \mathcal{N}(j)$
	\State Upon receiving ${w}_{i,j}^t$ from a worker $i$:
	\State \quad let $\tilde{w}_{i,j}\gets {w}_{i,j}^t$ 
	\State \quad update $\tilde z_{j}^{t+1}$ by \eqref{eq:update_z}.
	\State \quad if $w_{i,j}^t$ has been received for all $i \in \mathcal{N}(j)$ then $z_{j}^{t+1}\gets\tilde z_{j}^{t+1}$.
\end{algorithmic}
\end{algorithm}

Algorithm~\ref{alg:Asybadmm} describes the entire block-wise asynchronous distributed ADMM.
 Note that in Algorithm~\ref{alg:Asybadmm}, the block $j$ is randomly and independently selected from $\mathcal{N}(i)$ according to a uniform distribution, which is common in practice. Due to the page limit, we only consider the random block selection scheme, and we refer readers to other options including Gauss-Seidel and Gauss-Southwell block selection in the literature, e.g., \cite{hong2016unified}.


We put a few remarks on implementation issues at the end of this section to characterize key features of our proposed block-wise asynchronous algorithm, which differs from full-vector updates in the literature \cite{hong2017distributed}. \emph{Firstly}, model parameters are stored in blocks, so different workers can update different blocks asynchronously in parallel, which takes advantage of the popular Parameter Server architecture. \emph{Secondly}, workers can pull $\mathbf{z}$ while others are updating some blocks, enhancing concurrency. \emph{Thirdly}, in our implementation, workers will compute both gradients and local variables. In contrast, in the full-vector ADMM \cite{hong2017distributed}, workers are only responsible for computing gradients, therefore all previously computed and transmitted $\tilde{w}_{i,j}$ must be cached on servers with non-negligible memory overhead. 



%% file: src/theory.tex

\section{Convergence Analysis}
\label{sec:theory}

\begin{figure*}[t]
  \centering
  \subfigure[iteration vs. objective]{
    \includegraphics[width=3.0in]{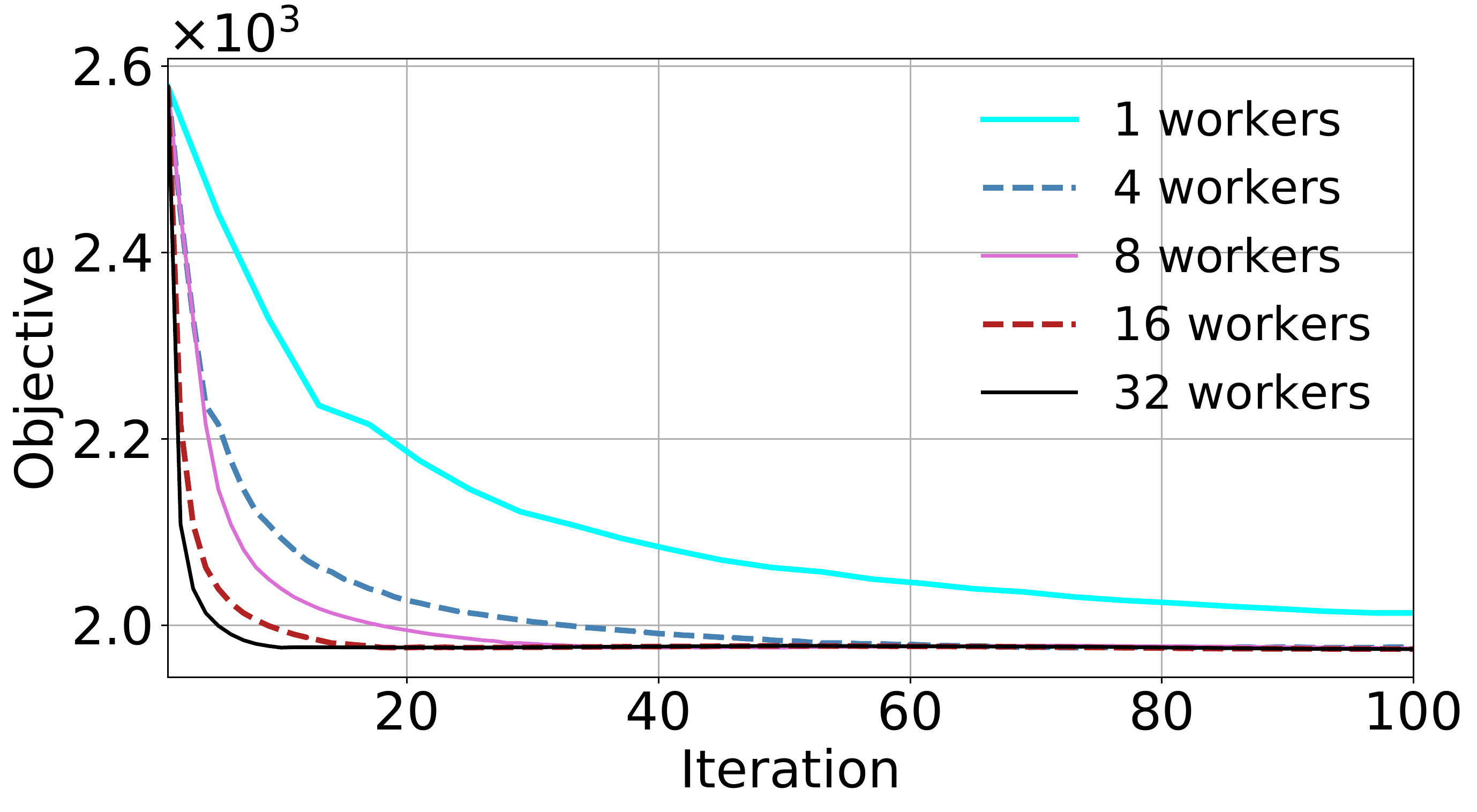}
    \label{fig:obj1}
  }
  \vspace{1mm}
  \subfigure[time vs. objective]{
    \includegraphics[width=3.0in]{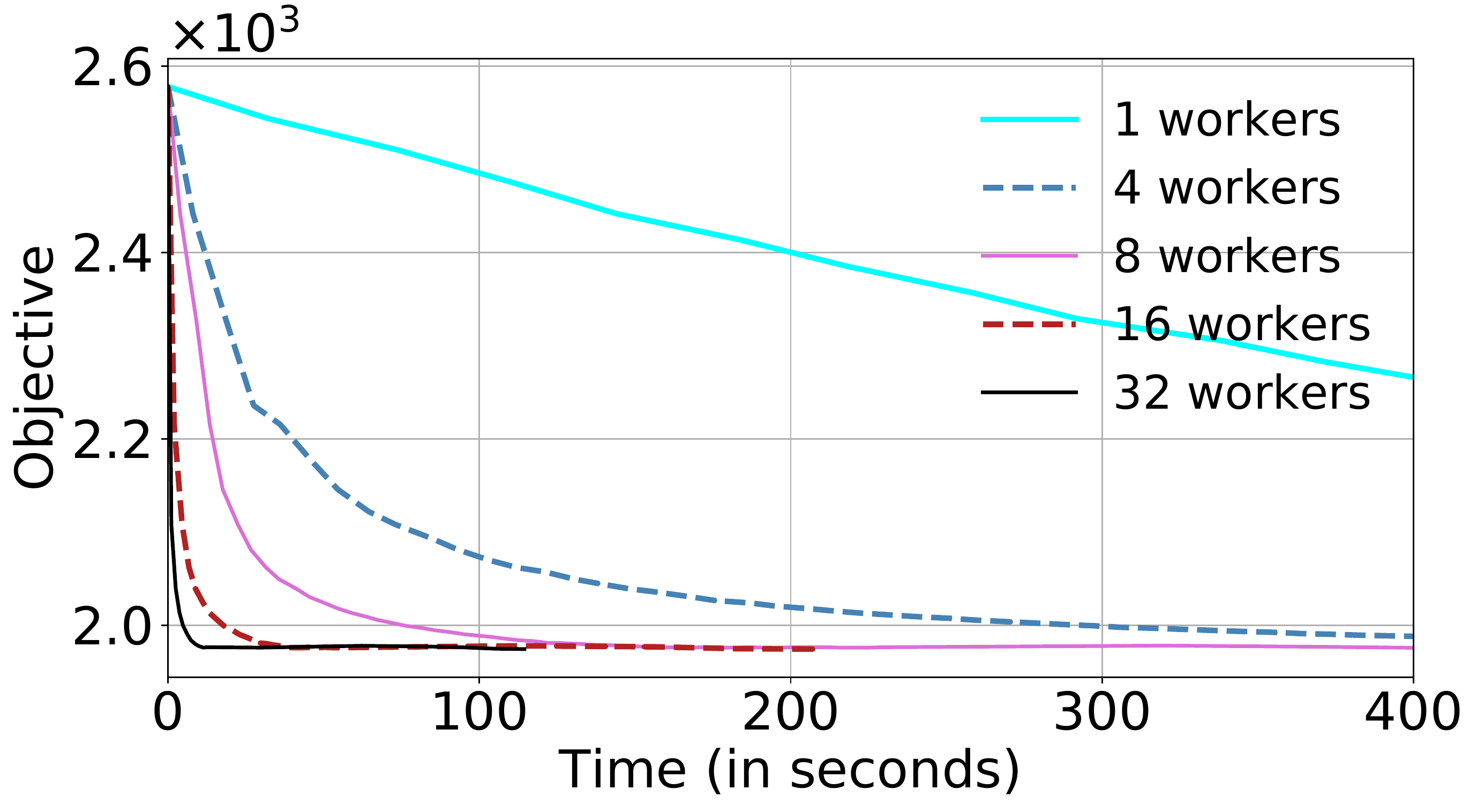}
    \label{fig:obj2}
  }
  \vspace{-3mm}
  \caption{Convergence of AsyBADMM on the sparse logistic regression problem.}
  \vspace{-2mm}
  \label{fig:simu_obj}
\end{figure*}

In this section, we provide convergence analysis of our algorithm under certain standard assumptions:
\begin{asmp}[Block Lipschitz Continuity]
For all $(i, j)\in\mathcal{E}$, there exists a positive constant $L_{i,j} > 0$ such that
\begin{displaymath}
	\norm{\nabla_j f_i(\mathbf{x}) - \nabla_j f_i(\mathbf{z})} \leq L_{i,j} \norm{x_j - z_j}, \forall \mathbf{x}, \mathbf{z} \in\real{d}.
\end{displaymath}
\vspace{-3mm}
\label{asmp:lipschitz}
\end{asmp}
\begin{asmp}[Bounded from Below]
Each function $f_i(\mathbf{x})$is bounded below, i.e., there exists a finite number $\underline{f} > -\infty$ where $\underline{f}$ denotes the optimal objective value of problem \eqref{eq:general-consensus}.
\label{asmp:below}
\end{asmp}
\begin{asmp}[Bounded Delay]
  The total delay of each link $(i,j)\in \mathcal{E}$ is bounded with the constant of $T_{i,j}$ for each pair of worker $i$ and server $j$. Formally, there is an integer $0 \leq \tau \leq T_{i,j}$, such that $\tilde{z}_j^t = z_j^{t-\tau}$ for all $t>0$. This should also hold for $\tilde{w}_{i,j}$.
\label{asmp:delay}
\end{asmp}

To characterize the convergence behavior, a commonly used metric is the squared norm of gradients. Due to potential nonsmoothness of $h(\cdot)$, Hong \emph{et al.} propose a new metric \shortcite{hong2016convergence} as a hybrid of gradient mapping as well as vanilla gradient as follows:
\begin{align}
P(\mathbf{X}^t, \mathbf{Y}^t, \mathbf{z}^t) :=& \norm{\mathbf{z}^t - \hat{\mathbf{z}}^{t}}^2 + \sum_{(i,j)\in\mathcal{E}} \norm{\nabla_{x_{i,j}} L(\mathbf{X}^t, \mathbf{Y}^t, \mathbf{z}^t)}^2 \nonumber \\
& + \sum_{(i,j)\in\mathcal{E}} \norm{x_{i,j} - z_j}^2,
\label{eq:grad_norm}
\end{align}
where $\hat{\mathbf{z}}^t=(\hat{z}_1^t, \ldots, \hat{z}_M^t)$ is defined as
\begin{equation}
	\hat{z}_j^t := \prox{h}(z_j^t - \nabla_{z_j}(L(\mathbf{X}^t, \mathbf{Y}^t, \mathbf{z}^t)-h(\mathbf{z}^t))).
\end{equation}
It is clear that if $P(\mathbf{X}^t, \mathbf{Y}^t, \mathbf{z}^t)\to 0$, then we obtain a stationary solution of \eqref{eq:original}.

The following Theorem~\ref{thm:convergence} indicates that Algorithm~\ref{alg:Asybadmm} converges to a stationary point satisfying KKT conditions under suitable choices of hyper-parameters.
\begin{thm}
Suppose that Assumptions~\ref{asmp:lipschitz}-\ref{asmp:delay} hold. Moreover, for all $i$ and $j$, the penalty parameter $\rho_{i}$ and $\gamma$ are chosen to be sufficiently large such that:
{\footnotesize
\begin{align}
\infty & > L(\mathbf{X}^0, \mathbf{Y}^0, \mathbf{z}^0) - \underline{f} \geq 0 \\
& \alpha_{j} := \gamma + \rho_i -\sum_{i\in \mathcal{N}(j)}\left(\frac{1}{2} + \frac{1}{\rho_{i}}\right) L_{i,j}^2(T_{i,j}+1)^2  \nonumber\\
&\quad\quad\quad - \sum_{i\in \mathcal{N}(j)}\frac{(4L_{i,j}+\rho_{i}+1)T_{i,j}^2}{2} >0, \\
& \beta_{i} := \frac{\rho_{i}- 4\max_{j \in \mathcal{N}(i)} L_{i,j}}{2|\mathcal{N}(i)|} > 0.
\end{align}
}
Then the following is true for Algorithm~\ref{alg:Asybadmm}:
\begin{enumerate}
	\item Algorithm~\ref{alg:Asybadmm} converges in the following sense:
	\begin{subequations}
	\begin{align}
	\lim_{t\to \infty} \norm{z_j^{t+1}  - z_j^{t}} = 0, &\quad \forall j=1,\ldots,M, \label{eq:lim_z}\\
	\lim_{t\to \infty} \norm{x_{i,j}^{t+1}  - x_{i,j}^{t}} = 0, &\quad \forall (i, j) \in \mathcal{E}, \label{eq:lim_x} \\
	\lim_{t\to \infty} \norm{y_{i,j}^{t+1}  - y_{i,j}^{t}} = 0, &\quad \forall (i, j) \in \mathcal{E}. \label{eq:lim_y}
	\end{align}	
	\end{subequations}

	\item For each worker $i$ and server $j$, denote the limit points of $\{x_{i,j}^t\}, \{y_{i,j}^t\}$, and $\{z_{j}^t\}$ by $x_{i,j}^*, y_{i,j}^*$ and $z_{j}^*$, respectively. Then these limit points satisfy KKT conditions, i.e., we have
	\begin{subequations}
		\begin{align}
			\nabla_j f_i(\mathbf{x}_i^*) + {y}_{i,j}^* = 0, &\quad \forall (i,j) \in \mathcal{E}, \label{eq:kkt-1}\\
			\sum_{j \in \mathcal{N}(i)} y_{i,j}^* \in \partial h_j(z_j^*) , &\quad \forall j=1,\ldots,M, \label{eq:kkt-2}\\
			x_{i,j}^{*} = z_j^* \in \mathcal{X}_j, &\quad \forall (i, j)\in \mathcal{E}. \label{eq:kkt-3}
		\end{align}
	\end{subequations}
	When sets $\mathcal{X}_j$ are compact, the sequence of iterates generated by Algorithm~\ref{alg:Asybadmm} converges to stationary points.

	\item For some $\epsilon > 0$, let $T(\epsilon)$ denote the epoch that achieves the following:
	\[ T(\epsilon) = \min \{ t | P(\mathbf{X}^t, \mathbf{Y}^t, \mathbf{z}^t) \leq \epsilon, t \geq 0  \}. \]
	Then there exists some constant $C > 0$ such that
	\begin{equation}
		T(\epsilon) \leq \frac{C(L(\mathbf{X}^0, \mathbf{Y}^0, \mathbf{z}^0) - \underline{f})}{\epsilon},
	\end{equation}
	where $\underline{f}$ is defined in Assumption~\ref{asmp:below}.
\end{enumerate}
\label{thm:convergence}
\end{thm}
Due to the non-convex objective function $f_i(\mathbf{x}_i)$, no guarantee of global optimality is possible in general. The parameter $\rho_{i}$ acts like the learning rate hyper-parameter in gradient descent: a large $\rho_{i}$ slows down the convergence and a smaller one can speed it up. The term $\gamma$ is associated with the delay bound $T_{i,j}$. In the synchronous case, we can set $\gamma=0$; otherwise, to guarantee convergence, $\gamma$ should be increased as the maximum allowable delay $T_{i,j}$ increases.

%% file: src/simu.tex

\section{Experiments}
\label{sec:simu}

We now show how our algorithm can be used to solve the challenging non-convex non-smooth problems in machine learning. We will show how AsyBADMM exhibits a near-linear speedup as the number of workers increases. We use a cluster of 18 instances of type \texttt{c4.large} on Amazon EC2. This type of instances has 2 CPU cores and at least 3.75 GB RAM, running 64-bit Ubuntu 16.04 LTS (HVM). Each server and worker process uses up to 2 cores. In total, our deployment uses 36 CPU cores and 67.5 GB RAM. Two machines serve as server nodes, while the other 16 machines serve as worker nodes. Note that we treat one core as a computational node (either a worker or server node).

\textbf{Setup:} 
In this experiment, we consider the sparse logistic regression problem:
\begin{equation}
\begin{split}
  \min_{\mathbf{x}} &\quad \frac{1}{m}\sum_{l=1}^{m} \log(1+\exp(-\tilde{y}_l \innerprod{\tilde{\mathbf{x}}_l, \mathbf{x}})) + \lambda \norm{\mathbf{x}}_1\\
  \mathrm{s.t.} &\quad \norm{\mathbf{x}}_\infty \leq C,
\end{split}
\end{equation}
where the constant $C$ is used to clip out some extremely large values for robustness. The $\ell_1$-regularized logistic regression is one of the most popular algorithms used for large scale risk minimization. We consider a public sparse text dataset \texttt{KDDa} \footnote{\url{http://www.csie.ntu.edu.tw/~cjlin/libsvmtools/datasets/}}. This dataset has more than 8 million samples, 20 million features, and 305 million nonzero entries. To show the advantage of parallelism, we set up five experiments with 1, 4, 8, 16 and 32 nodes, respectively. In each experiment, the whole dataset will be evenly split into several smaller parts, and each node only has access to its local dataset.

We implement our algorithm on the \texttt{ps-lite} framework \cite{li2014scaling}, which is a lightweight implementation of Parameter Server architecture. It supports Parameter Server for multiple devices in a single machine, and multiple machines in a cluster. This is the back end of \texttt{kvstore} API of the deep learning framework MXNet \cite{chen2015mxnet}. Each worker updates the blocks by cycling through the coordinates of $x$ and updating each in turns, restarting at a random coordinate after each cycle.


\textbf{Results:}
Empirically, Assumption~\ref{asmp:delay} is observed to hold for this cluster. We set the hyper-parameter $\gamma=0.01$ and the clip threshold constant as $C=10^4$, and the penalty parameter $\rho_{i,j}=100$ for all $(i,j)$. Fig.~\ref{fig:obj1} and Fig.~\ref{fig:obj2} show the convergence behavior of our proposed algorithm in terms of objective function values. From the figures, we can clearly observe the convergence of our proposed algorithm. This observation confirms that asynchrony with tolerable delay can still lead to convergence.

To further analyze the parallelism in AsyBADMM, we measure the speedup by the relative time for $p$ workers to perform $k$ iterations, i.e., Speedup of $p$ workers = $\frac{T_k(1)}{T_k(p)}$,
where $T_k(p)$ is the time it takes for $p$ workers to perform $k$ iterations of optimization. Fig.~\ref{fig:obj2} illustrates the running time comparison and Table~\ref{table:time_vs_iter} shows that AsyBADMM actually achieves near-linear speedup.

\begin{table}[]
\centering
\caption{Running time (in seconds) for iterations $k$ and worker count.}
\label{table:time_vs_iter}
\begin{tabular}{ccccc}
\hline
Workers $p$ & $k=20$ & $k=50$ & $k=100$ & Speedup \\ \hline
1       & 1404   & 3688   & 6802    & 1.0     \\
4       & 363    & 952    & 1758    & 3.87    \\
8       & 177    & 466    & 859     & 7.92    \\
16      & 86     & 226    & 417     & 16.31   \\
32      & 47     & 124    & 228     & 29.83   \\ \hline
\end{tabular}
\end{table}



%% file: src/conclude.tex

\section{Concluding Remarks}
\label{sec:conclude}

In this paper, we propose a block-wise, asynchronous and distributed ADMM algorithm to solve general non-convex and non-smooth optimization problems in machine learning. Under the bounded delay assumption, we have shown that our proposed algorithm can converge to stationary points satisfying KKT conditions. The block-wise updating nature of our algorithm makes it feasible to be implemented on Parameter Server, take advantage of the ability to update different blocks of all model parameters in parallel on distributed servers.
Experimental results based on a real-world dataset have demonstrated the convergence and near-linear speedup of the proposed ADMM algorithm, for training large-scale sparse logistic regression models in Amazon EC2 clusters.

%% file: src/appendix.tex




\onecolumn

\appendix

\section{Key Lemmas}
\begin{lem}
Suppose Assumption~\ref{asmp:lipschitz}-\ref{asmp:delay} are satisfied. Then we have
\begin{equation}
\norm{y_{i,j}^{t+1} - y_{i,j}^t }^2 \leq L_{i,j}^2 (T_{i,j} + 1) \sum_{t'=0}^{T_{i,j}} \norm{z_j^{t+1-t'} - z_j^{t-t'}}^2.
\label{eq:ybound}
\end{equation}
\label{lem:ybound}
\end{lem}
\begin{proof}
For simplicity, we say $(i,j)$ is performed at epoch $t$ when worker $i$ is updating block $j$ at epoch $t$. If the updating $(i,j)$ is not performed at epoch $t$, this inequality holds in trivial, as $y_{i,j}^{t+1} = y_{i,j}^t$. So we only consider the case that $(i,j)$ is performed at epoch $t$. Note that in this case, we have
\begin{equation}
	\nabla_j f_i(\tilde{\mathbf{z}}^{t+1}) + y_{i,j}^t + \rho_{i} (x_{i,j}^{t+1} - \tilde{z}_{j}^{t+1}) = 0.
\end{equation}
Since $y_{i,j}^{t+1} = y_{i,j}^t + \rho_{i}(x_{i,j}^{t+1} - \tilde{z}_{j}^{t+1})$, we have
\begin{equation}
	\nabla_j f_i(\tilde{\mathbf{z}}^{t+1}) + y_{i,j}^{t+1} = 0.
\end{equation}
Therefore, we have
\begin{equation*}
\begin{split}
	\norm{y_{i,j}^{t+1} - y_{i,j}^t } &= \norm{\nabla_j f_i(\tilde{\mathbf{z}}^{t+1}) - \nabla_j f_i(\tilde{\mathbf{z}}^{t})} \\
	&\leq L_{i,j} \norm{\tilde{z}_j^{t+1} - \tilde{z}_j^{t}}.
\end{split}
\end{equation*}
Since the actual updating time for $\tilde{z}_j^{t+1}$ should be in $\{t+1, t, \ldots, t+1-T_{i,j}\}$, and for $\tilde{z}_j^{t}$ in $\{t, \ldots, t-T_{i,j}\}$, then we have
\begin{equation}
\norm{y_{i,j}^{t+1} - y_{i,j}^t }^2 \leq L_{i,j}^2 (T_{i,j} + 1) \sum_{t'=0}^{T_{i,j}} \norm{z_j^{t+1-t'} - z_j^{t-t'}}^2,
\end{equation} 
which proves the lemma.
\end{proof}

\begin{lem}
	At epoch $t$, we have
	\begin{equation}
		\nabla_j f_{i}(\tilde{\mathbf{z}}^{t+1}) + y_{i,j}^t + \rho_{i}(x_{i,j}^{t+1} - {z}_{j}^{t+1}) = \rho_{i} (\tilde{z}_{j}^{t+1} - z_{j}^{t+1}).
	\end{equation}
	\label{lem:4}
\end{lem}
\begin{proof}
Updating $x_{i,j}^{t+1}$ is performed as follows
\begin{equation*}
	x_{i,j}^{t+1} = \mathop{\arg\min}_{x_{i,j}} f_i(\tilde{\mathbf{z}}^{t+1}) + \innerprod{\nabla_j f_{i}(\tilde{\mathbf{z}}^{t+1}), x_{i,j} - \tilde{z}_{j}^{t+1}} + \innerprod{y_{i,j}^t, x_{i,j} - z_{j}^{t+1}} + \frac{\rho_{i}}{2}\norm{x_{i,j} - z_{j}^{t+1}}^2.
\end{equation*}
Thus, we have
\begin{equation*}
	\nabla_j f_{i}(\tilde{\mathbf{z}}^{t+1}) + y_{i,j}^t + \rho_{i}(x_{i,j}^{t+1} - \tilde{z}_{j}^{t+1}) = 0,
\end{equation*}
And therefore
\begin{equation*}
	\nabla_j f_{i}(\tilde{\mathbf{z}}^{t+1}) + y_{i,j}^t + \rho_{i}(x_{i,j}^{t+1} - {z}_{j}^{t+1}) = \rho_{i} (\tilde{z}_{j}^{t+1} - z_{j}^{t+1}).
\end{equation*}
\end{proof}

\begin{lem}
	Suppose Assumption~\ref{asmp:lipschitz}-\ref{asmp:delay} are satisfied. Then we have
\begin{eqnarray}
&\quad& L(\mathbf{X}^{T}, \mathbf{Y}^{T}, \mathbf{z}^{T}) - L(\mathbf{X}^0, \mathbf{Y}^0, \mathbf{z}^0) \\
&\leq& - \sum_{t=0}^{T-1}\sum_{(i, j) \in \mathcal{E}}\beta_{i} (\norm{x_{i,j}^{t+1}-\tilde{z}_j^{t+1}}^2 + \norm{x_{i,j}^{t} - {z}_j^{t+1s} }^2) - \sum_{t=0}^{T-1}\sum_{(i,j)\in\mathcal{E}} \alpha_{j} \norm{z_{j}^{t+1} - z_{j}^{t}}^2,
\label{eq:xzbound}
\end{eqnarray}
where
\begin{align}
	\alpha_j &:= (\gamma + \rho_i) - \sum_{i\in\mathcal{N}(j)}\left(\frac{1}{\rho_i} + \frac{1}{2} \right) L_{i,j}^2 (T_{i,j}+1)^2  - \sum_{i \in \mathcal{N}(j)} \frac{(4L_{i,j}+\rho_i+1)T_{i,j}^2}{2}, \\
	\beta_i &:= \frac{\rho_i - \max_{j \in \mathcal{N}(i)} 4L_{i,j}}{2|\mathcal{N}(i)|}.
\end{align}
\label{lem:iteration}
\end{lem}

\begin{lem}
	Suppose that Assumption~\ref{asmp:lipschitz}-\ref{asmp:delay} hold. Then the sequence of solutions $\{\mathbf{X}^t, \mathbf{z}^t, \mathbf{Y}^t\}$ satisfies
	\begin{equation}
		\lim_{t \to \infty}L(\mathbf{X}^t, \mathbf{z}^t, \mathbf{Y}^t) \geq \underline{f} - \mathrm{diam}^2(\mathcal{X}) \sum_{(i,j)\in \mathcal{E}} \frac{L_{i,j}}{2} > -\infty.
	\end{equation}
	\label{lem:boundedbelow}
\end{lem}

\section{Proof of Lemma~\ref{lem:iteration}}
Next we try to bound the gap between two consecutive Augmented Lagrangian values and break it down into three steps, namely, updating $\mathbf{X}$, $\mathbf{Y}$ and $\mathbf{z}$:
\begin{equation}
\begin{split}
& L(\mathbf{X}^{t+1}, \mathbf{Y}^{t+1}, \mathbf{z}^{t+1}) - L(\mathbf{X}^t, \mathbf{Y}^t, \mathbf{z}^t) \\
&= L(\mathbf{X}^{t}, \mathbf{Y}^{t}, \mathbf{z}^{t+1}) - L(\mathbf{X}^t, \mathbf{Y}^t, \mathbf{z}^t) \\
&\quad + L(\mathbf{X}^{t+1}, \mathbf{Y}^{t}, \mathbf{z}^{t+1}) - L(\mathbf{X}^{t}, \mathbf{Y}^t, \mathbf{z}^{t+1}) \\
&\quad + L(\mathbf{X}^{t+1}, \mathbf{Y}^{t+1}, \mathbf{z}^{t+1}) - L(\mathbf{X}^{t+1}, \mathbf{Y}^{t}, \mathbf{z}^{t+1}).
\end{split}
\end{equation}

To prove Lemma~\ref{lem:iteration}, we bound the above three gaps individually. Firstly, we bound the on $\mathbf{X}$. For each worker $i$, at epoch $t$, we use the following auxiliary function for convergence analysis:
\begin{equation}
	l_i(\mathbf{x}_i, y_{i,j}, z_j) := f_i(\mathbf{x}_i) +  \innerprod{y_{i,j}, x_{i,j} - z_{j}} + \frac{\rho_{i}}{2}\norm{x_{i,j} - z_{j}}^2 \label{eq:lem2-tmp-1}
\end{equation}
To simplify our proof in this section, we consider the case that only one block is updated. Therefore, only block $j$ in $\mathbf{x}_i^{t+1}$ differs from $\mathbf{x}_i^{t}$, and similarly for $\mathbf{y}_i^{t+1}$ and $\mathbf{z}^{t+1}$. We will use $\tilde{\mathbf{z}}^{t}:=\mathbf{z}^{t(i,j)}$ as the delayed version of $\mathbf{z}$ in this proof.

\begin{lem}
	For node $i$, we have the following inequality to bound the gap after updating ${x}_{i,j}$:
	\begin{equation}
	\begin{split}
	l_i(\mathbf{x}_i^{t+1}, y_{i,j}^t, z_j^{t+1}) - l_i(\mathbf{x}_i^t, y_{i,j}^t, z_j^{t+1}) 
	&\leq \frac{L_{i,j}}{2}\norm{x_{i,j}^{t+1}-\tilde{z}_j^{t+1}}^2 + \frac{L_{i,j}+\rho_{i}}{2} \norm{\tilde{z}_j^{t+1} - z_j^{t+1}}^2 + \frac{L_{i,j}}{2} \norm{x_{i,j}^{t+1} - {x}_{i,j}^{t}}^2  \\
	&\quad - \frac{\rho_{i}}{2}\left( \norm{x_{i,j}^{t} - z_j^{t+1} }^2
	 + \norm{x_{i,j}^{t+1} - \tilde{z}_j^{t+1} }^2  \right).
	\end{split}
	\end{equation}
	\label{lem:diff1-1}
\end{lem}

\begin{proof}
	From the block Lipschitz assumption and the updating rule, we have
\begin{align}
	f_i(\mathbf{x}_i^{t+1}) \leq f_i(\mathbf{x}^{t}) + \innerprod{\nabla_j f_i(\mathbf{x}^{t}), {x}_{i,j}^{t+1} - {x}_{i,j}^{t}}
	+ \frac{L_{i,j}}{2} \norm{x_{i,j}^{t+1} - {x}_{i,j}^{t}}^2.
\end{align}
By the definition of $\l_i(\cdot)$, we have
\begin{align}
	l_i(\mathbf{x}_i^{t+1}, y_{i,j}^t, z_j^{t+1}) 
	&\leq l_i(\mathbf{x}_i^t, y_{i,j}^t, z_j^{t+1}) + \innerprod{\nabla_j f_i(\mathbf{x}^{t}), {x}_{i,j}^{t+1} - {x}_{i,j}^{t}}
	+ \frac{L_{i,j}}{2} \norm{x_{i,j}^{t+1} - {x}_{i,j}^{t}}^2 + \innerprod{y_j^t, x_{i,j}^{t+1} - x_{i,j}^t} \nonumber \\
	&\quad + \frac{\rho_{i}}{2}\norm{x_{i,j}^{t+1} - z_j^{t+1} }^2 - \frac{\rho_{i}}{2}\norm{x_{i,j}^{t} - z_j^{t+1} }^2 \\
	&=  l_i(\mathbf{x}_i^t, y_{i,j}^t, z_j^{t+1})  + \frac{L_{i,j}}{2} \norm{x_{i,j}^{t+1} - {x}_{i,j}^{t}}^2 + \innerprod{\nabla_j f_i(\mathbf{x}^{t})+y_j^t, {x}_{i,j}^{t+1} - {x}_{i,j}^{t}} \nonumber \\
	&\quad + \frac{\rho_{i}}{2}\left( \norm{x_{i,j}^{t+1} - z_j^{t+1} }^2 - \norm{x_{i,j}^{t} - z_j^{t+1} }^2 \right).
\end{align}
The right hand side is actually a quadratic function w.r.t. $x_{i,j}^{t+1}$. Therefore, it is strongly convex, and we have
\begin{align}
	l_i(\mathbf{x}_i^{t+1}, y_{i,j}^t, z_j^{t+1}) 
	&\leq l_i(\mathbf{x}_i^t, y_{i,j}^t, z_j^{t+1}) + \frac{L_{i,j}}{2} \norm{x_{i,j}^{t+1} - {x}_{i,j}^{t}}^2 + \frac{\rho_{i}}{2}\norm{\tilde{z}_j^{t+1} - z_j^{t+1}}^2 - \frac{\rho_{i}}{2}\norm{x_{i,j}^{t} - z_j^{t+1} }^2 \nonumber \\
	&\quad + \innerprod{\nabla_j f_i(\mathbf{x}_i^{t+1})+y_j^t + \rho_{i}(x_{i,j}^{t+1}-z_j^{t+1}), \tilde{z}_j^{t+1} - x_{i,j}^t} - \frac{\rho_{i}}{2} \norm{x_{i,j}^{t+1} - \tilde{z}_j^{t+1} }^2 
\end{align}
By Lemma~\ref{lem:4}, we have
\begin{align*}
	\nabla_j f_i(\mathbf{x}_i^{t+1})+y_j^t + \rho_{i}(x_{i,j}^{t+1}-z_j^{t+1})
 	=& \nabla_j f_i(\tilde{\mathbf{z}}^{t+1})+y_j^t + \rho_{i}(x_{i,j}^{t+1}-z_j^{t+1}) + (\nabla_j f_i(\mathbf{x}_i^{t+1}) - \nabla_j f_i(\tilde{\mathbf{z}}^{t+1})) \\
 	=& \rho_{i}(\tilde{z}_j^{t+1} - z_j^{t+1}) + (\nabla_j f_i(\mathbf{x}_i^{t+1}) - \nabla_j f_i(\tilde{\mathbf{z}}^{t+1})).
\end{align*}
Therefore, we have
\begin{align*}
	l_i(\mathbf{x}_i^{t+1}, y_{i,j}^t, z_j^{t+1}) 
	&\leq l_i(\mathbf{x}_i^t, y_{i,j}^t, z_j^{t+1}) + \innerprod{\nabla_j f_i(\mathbf{x}_i^{t+1}) - \nabla_j f_i(\tilde{\mathbf{z}}^{t+1}), \tilde{z}_j^{t+1} - z_j^{t+1}} \\
	&\quad + \frac{L_{i,j}}{2} \norm{x_{i,j}^{t+1} - {x}_{i,j}^{t}}^2 - \frac{\rho_{i}}{2}\norm{x_{i,j}^{t} - z_j^{t+1} }^2 - \frac{\rho_{i}}{2} \norm{x_{i,j}^{t+1} - \tilde{z}_j^{t+1} }^2 + \frac{\rho_{i}}{2}\norm{\tilde{z}_j^{t+1} - z_j^{t+1}}^2 \\
	&\leq l_i(\mathbf{x}_i^t, y_{i,j}^t, z_j^{t+1}) + \frac{L_{i,j}}{2}\norm{x_{i,j}^{t+1}-\tilde{z}_j^{t+1}}^2 + \frac{L_{i,j}+\rho_{i}}{2} \norm{\tilde{z}_j^{t+1} - z_j^{t+1}}^2 + \frac{L_{i,j}}{2} \norm{x_{i,j}^{t+1} - {x}_{i,j}^{t}}^2 \nonumber \\
	&\quad - \frac{\rho_{i}}{2}\left( \norm{x_{i,j}^{t} - z_j^{t+1} }^2
	 + \norm{x_{i,j}^{t+1} - \tilde{z}_j^{t+1} }^2  \right) \\
	&\leq l_i(\mathbf{x}_i^t, y_{i,j}^t, z_j^{t+1}) + \frac{4L_{i,j}}{2}\norm{x_{i,j}^{t+1}-\tilde{z}_j^{t+1}}^2 + \frac{4L_{i,j}+\rho_{i}}{2} \norm{\tilde{z}_j^{t+1} - z_j^{t+1}}^2 +\frac{3L_{i,j}}{2} \norm{x_{i,j}^t - z_j^{t+1}} \nonumber \\
	&\quad - \frac{\rho_{i}}{2}\left( \norm{x_{i,j}^{t} - z_j^{t+1} }^2
	 + \norm{x_{i,j}^{t+1} - \tilde{z}_j^{t+1} }^2  \right) \\
	&\leq l_i(\mathbf{x}_i^t, y_{i,j}^t, z_j^{t+1}) - \left( \frac{\rho_{i}}{2}- \frac{4L_{i,j}}{2} \right)\norm{x_{i,j}^{t+1}-\tilde{z}_j^{t+1}}^2 - \left(\frac{\rho_{i}}{2} - \frac{3L_{i,j}}{2} \right) \norm{x_{i,j}^{t} - {z}_j^{t+1} }^2  \\
	&\quad + \frac{4L_{i,j}+\rho_{i}}{2} \norm{\tilde{z}_j^{t+1} - z_j^{t+1}}^2 
\end{align*}
\end{proof}

\begin{cor}
If block $j$ is randomly drawn from uniform distribution, we have
\begin{equation}
\begin{split}
	\mathbb{E}_{j}[L_i(\mathbf{x}_i^{t+1}, \mathbf{y}_i^t, \mathbf{z}^{t+1})]
	&\leq \mathbb{E}_{j}[L_i(\mathbf{x}_i^{t}, \mathbf{y}_i^t, \mathbf{z}^{t+1})]
    - \frac{1}{|\mathcal{N}(i)|}\sum_{j\in \mathcal{N}(i)} \left( \frac{\rho_{i}}{2}- \frac{4L_{i,j}}{2} \right)\norm{x_{i,j}^{t+1}-\tilde{z}_j^{t+1}}^2 \\
    &\quad - \frac{1}{|\mathcal{N}(i)|}\sum_{j\in \mathcal{N}(i)} \left(\frac{\rho_{i}}{2} - \frac{3L_{i,j}}{2} \right) \norm{x_{i,j}^{t} - {z}_j^{t+1} }^2 + \frac{1}{|\mathcal{N}(i)|}\sum_{j\in \mathcal{N}(i)} \frac{4L_{i,j}+\rho_{i}}{2} \norm{\tilde{z}_j^{t+1} - z_j^{t+1}}^2 
\end{split}
\label{eq:diff_x_whole}
\end{equation}
\label{cor:diff1-1}
\end{cor}

\begin{lem}
	For node i, we have the following inequality to bound the gap after updating $\mathbf{y}_i$:
	\[l_i(\mathbf{x}_i^{t+1}, y_{i,j}^t, z_j^{t+1}) - l_i(\mathbf{x}_i^{t}, y_{i,j}^t, z_j^{t+1}) \leq \left(\frac{1}{\rho_{i}} + \frac{1}{2} \right) \norm{y_{i,j}^{t+1} - y_{i,j}^{t}}^2 + \frac{1}{2}\norm{\tilde{z}_{j}^{t+1} - z_{j}^{t+1}}^2.\]
	\label{lem:diff1-2}
\end{lem}
\begin{proof}
From \eqref{eq:update_y} we have
\begin{align*}
	l_i(\mathbf{x}_i^{t+1}, y_{i,j}^t, z_j^{t+1}) - l_i(\mathbf{x}_i^{t}, y_{i,j}^t, z_j^{t+1}) 
	&= \innerprod{y_{i,j}^{t+1}-y_{i,j}^{t}, x_{i,j}^{t+1}-z_{j}^{t+1}} \\
	&= \innerprod{y_{i,j}^{t+1}-y_{i,j}^{t}, x_{i,j}^{t+1}-\tilde{z}_{j}^{t+1}} + \innerprod{y_{i,j}^{t+1}-y_{i,j}^{t}, \tilde{z}_{j}^{t+1}-z_{j}^{t+1}} \\
	&\leq \frac{1}{\rho_{i}}\norm{y_{i,j}^{t+1}-y_{i,j}^{t}}^2 + \frac{1}{2}\norm{y_{i,j}^{t+1} - y_{i,j}^t}^2 + \frac{1}{2}\norm{\tilde{z}_{j}^{t+1} - z_j^{t+1}}^2 \\
	&= \left( \frac{1}{\rho_{i}} + \frac{1}{2} \right)\norm{y_{i,j}^{t+1}-y_{i,j}^{t}}^2 + \frac{1}{2}\norm{\tilde{z}_{j}^{t+1} - z_{j}^{t+1}}^2.
\end{align*}
\end{proof}

\begin{cor}
If block $j$ is randomly drawn from uniform distribution, we have
\begin{equation}
\begin{split}
	&\quad \mathbb{E}_{j}[L_i(\mathbf{x}_i^{t+1}, \mathbf{y}_i^{t+1}, \mathbf{z}^{t+1})]
	- \mathbb{E}_{j}[L_i(\mathbf{x}_i^{t+1}, \mathbf{y}_i^t, \mathbf{z}^{t+1})] \\
	&\leq \frac{1}{|\mathcal{N}(i)|} \left( \frac{1}{\rho_{i}} + \frac{1}{2} \right) \sum_{j\in\mathcal{N}(i)}\norm{y_{i,j}^{t+1}-y_{i,j}^{t}}^2 + \frac{1}{2|\mathcal{N}(i)|}\sum_{j\in\mathcal{N}(i)} \norm{\tilde{z}_{j}^{t} - z_{j}^{t}}^2
\end{split}
\label{eq:diff_y_whole}
\end{equation}
\label{cor:diff1-2}
\end{cor}

\begin{lem}
	After updating $z_j^t$ to $z_j^{t+1}$, we have
	\begin{equation}
	\mathbb{E}_j[l(\mathbf{X}^{t}, \mathbf{Y}^{t}, {z}_j^{t+1})] - \mathbb{E}_j[l(\mathbf{X}^{t}, \mathbf{Y}^{t}, {z}_j^t)]
	\leq - \sum_{i \in \mathcal{N}(j)}\frac{\gamma + \rho_{i}}{|\mathcal{N}(i)|} \cdot \norm{z_{j}^{t+1} - z_{j}^{t}}^2.
	\label{eq:diff_z_whole}
	\end{equation}
	\label{lem:diff1-3}
\end{lem}
\begin{proof}
We begin our proof by analyzing the block $j$. Let
	\[
	l(\mathbf{X}, \mathbf{Y}, {z}_j):= h_j(z_j) + \sum_{i\in\mathcal{N}(j)} \innerprod{y_{i,j}, x_{i,j} - z_j} + \sum_{i\in\mathcal{N}(j)} \frac{\rho_{i}}{2}\norm{x_{i,j} - z_j}^2.
	\]
Firstly, it is clear that $\innerprod{y_{i,j}, x_{i,j}-z_j} + \rho_{i}\norm{x_{i,j}-z_j}^2$ is a quadratic function and thus strongly convex. Then, we have:
\begin{align*}
&\quad \sum_{i\in \mathcal{N}(j)}\innerprod{y_{i,j}^{t}, x_{i,j}^{t}-z_j^{t+1}} + \frac{\rho_{i}}{2}\norm{x_{i,j}^{t}-z_j^{t+1}}^2
- \sum_{i\in \mathcal{N}(j)}\innerprod{y_{i,j}^{t}, x_{i,j}^{t}-z_j^{t}} - \frac{\rho_{i}}{2}\norm{x_{i,j}^{t}-z_j^{t}}^2 \\
&\leq \innerprod{-y_{i,j}^{t} - \rho_i(x_{i,j}^{t}-z_j^{t+1}), z_{j}^{t+1}-z_{j}^t} - \sum_{i\in\mathcal{N}(j)}\frac{\rho_i}{2}\norm{z_{j}^{t+1} - z_j^t}^2.
\end{align*}
By the optimality in \eqref{eq:update_z}, we have
\begin{align*}
	\innerprod{p_j^{t+1} - \sum_{i\in\mathcal{N}(j)} y_{i,j}^{t} + \rho_i(z_j^{t+1} - x_{i,j}^{t}) + \gamma(z_j^{t+1} - z_j^t), z_j^{t+1} - z_j^{t}} \leq 0,
\end{align*}
where $p_j^{t+1} \in \partial h_j(z_j^{t+1})$ is a subgradient. By convexity of $h_j$, we have
\begin{align*}
	h_j(z_j^{t+1}) -h_j(z_j^{t}) &\leq  \innerprod{p_j^{t+1}, z_j^{t+1} - z_j^{t}} \\
	&\leq \innerprod{\sum_{i\in\mathcal{N}(j)} y_{i,j}^{t} - \rho_i(z_j^{t+1} - x_{i,j}^{t}) - \gamma(z_j^{t+1} - z_j^t), z_j^{t+1} - z_j^{t}}
\end{align*}
Therefore, by taking expectation on $j$, we have
\begin{align*}
	&\quad \mathbb{E}_j \left[l(\mathbf{X}^{t}, \mathbf{Y}^{t}, {z}_j^{t+1})+\frac{\gamma}{2}\norm{z_{j}^{t+1}-z_j^t}^2 \right] - \mathbb{E}_j[l(\mathbf{X}^{t}, \mathbf{Y}^{t}, {z}_j^t)] \\
	&\leq \innerprod{- \sum_{i\in\mathcal{N}(j)}\frac{1}{|\mathcal{N}(i)|}(y_{i,j}^{t} - \rho_i(x_{i,j}^{t}-z_j^{t+1})), z_{j}^{t+1}-z_{j}^t} - \sum_{i\in\mathcal{N}(j)}\frac{\rho_i}{2|\mathcal{N}(i)|}\norm{z_{j}^{t+1} - z_j^t} \\
	&\quad + \innerprod{\sum_{i\in\mathcal{N}(j)} \frac{1}{|\mathcal{N}(i)|}[y_{i,j}^{t} - \rho_i(z_j^{t+1} - x_{i,j}^{t}) - \gamma(z_j^{t+1} - z_j^t)], z_j^{t+1} - z_j^{t}} \\
	&= -\sum_{i \in \mathcal{N}(j)}\frac{\gamma + 2\rho_{i}}{2 |\mathcal{N}(i)|} \cdot \norm{z_{j}^{t+1} - z_{j}^{t}}^2.
\end{align*}
which proves the lemma.
\end{proof}

We now proceed to prove Lemma~\ref{lem:iteration}. From Corollary~\ref{cor:diff1-1}--\ref{cor:diff1-2} and Lemma~\ref{lem:diff1-3}, we have three upper bounds when updating $x_{i,j}^t$, $y_{i,j}^t$ and $z_{j}^t$, respectively, and we observe that the sign of $\norm{x_{i,j}^t-z_j^t}$ can be negative by assuming $\rho_i \geq 3L_{i,j}$, and similarly for $\norm{x_{i,j}^{t+1} - \tilde{z}_j^t}$ by assuming $\rho_i \geq 4L_{i,j} \geq 0$. Therefore, let
\[ \rho_i - 4 \max_{j \in \mathcal{N}(i)} L_{i,j} \geq 0,\]
and then we can guarantee that the efficients for all $(i,j) \in \mathcal{E}$, the coefficients for $\norm{x_{i,j}^t-z_j^t}$ and $\norm{x_{i,j}^{t+1} - \tilde{z}_j^t}$ are always negative.

Then, the major challenge is to make the coefficient of $\norm{z_{j}^{t+1} - z_{j}^t}$ be negative, and we attempt to make it as follows:
\begin{align*}
	\frac{1}{|\mathcal{N}(i)|}\sum_{j \in \mathcal{N}(i)} \left( \frac{1}{\rho_i} + \frac{1}{2} \right)\norm{y_{i,j}^{t+1} - y_{i,j}^{t}}^2 
	&\leq \frac{1}{|\mathcal{N}(i)|}\sum_{j \in \mathcal{N}(i)} \left( \frac{1}{\rho_i} + \frac{1}{2} \right)L_{i,j}^2(T_{i,j}+1)\sum_{t'=0}^{T_{i,j}}\norm{z_j^{t-t'} - z_j^{t-t'-1}}^2, \\
	\frac{1}{|\mathcal{N}(i)|}\sum_{j \in \mathcal{N}(i)} \frac{4L_{i,j}+\rho_i+1}{2} \norm{\tilde{z}_{j}^{t+1} - z_{j}^{t+1}}^2
	&\leq \frac{1}{|\mathcal{N}(i)|}\sum_{j \in \mathcal{N}(i)} \frac{4L_{i,j}+\rho_i+1}{2}\cdot T_{i,j}\sum_{t'= 0}^{T_{i,j}-1}\norm{z_j^{t+1-t'} - z_j^{t-t'} }^2.
\end{align*}
We now combine equations \eqref{eq:diff_x_whole}, \eqref{eq:diff_y_whole} and \eqref{eq:diff_z_whole}, and sum over all workers $i$:
\begin{align*}
&\quad\ \  \mathbb{E}_j[L(\mathbf{X}^{t+1}, \mathbf{Y}^{t+1}, \mathbf{z}^{t+1})] - \mathbb{E}_j[L(\mathbf{X}^t, \mathbf{Y}^t, \mathbf{z}^t)] \\
&\leq - \sum_{(i, j) \in \mathcal{E}} \frac{1}{|\mathcal{N}(i)|} \left( \frac{\rho_{i}}{2}- \frac{4L_{i,j}}{2} \right)\norm{x_{i,j}^{t+1}-\tilde{z}_j^{t+1}}^2 - \sum_{(i, j) \in \mathcal{E}} \frac{1}{|\mathcal{N}(i)|} \left(\frac{\rho_{i}}{2} - \frac{3L_{i,j}}{2} \right) \norm{x_{i,j}^{t} - {z}_j^{t+1} }^2 \\
&\quad + \sum_{(i, j) \in \mathcal{E}} \frac{4L_{i,j}+\rho_{i}+1}{2|\mathcal{N}(i)|}\cdot \norm{\tilde{z}_j^{t+1} - z_j^{t+1}}^2 + \frac{1}{|\mathcal{N}(i)|} \left( \frac{1}{\rho_{i}} + \frac{1}{2} \right) \norm{y_{i,j}^{t+1}-y_{i,j}^{t}}^2 - \sum_{(i,j)\in\mathcal{E}}\frac{\gamma + \rho_{i}}{|\mathcal{N}(i)|} \cdot \norm{z_{j}^{t+1} - z_{j}^{t}}^2 \\
&\leq - \sum_{(i, j) \in \mathcal{E}} \frac{\rho_i - 4L_{i,j}}{2|\mathcal{N}(i)|} (\norm{x_{i,j}^{t+1}-\tilde{z}_j^{t+1}}^2 + \norm{x_{i,j}^{t} - {z}_j^{t+1} }^2) - \sum_{(i,j)\in\mathcal{E}}\frac{\gamma + \rho_{i}}{|\mathcal{N}(i)|} \cdot \norm{z_{j}^{t+1} - z_{j}^{t}}^2 \\
&\quad + \sum_{(i,j)\in\mathcal{E}} \frac{1}{|\mathcal{N}(i)|} \left( \frac{1}{\rho_i} + \frac{1}{2} \right)L_{i,j}^2(T_{i,j}+1)\sum_{t'=0}^{T_{i,j}}\norm{z_j^{t+1-t'} - z_j^{t-t'}}^2 \\
&\quad + \sum_{(i,j)\in\mathcal{E}}\frac{4L_{i,j}+\rho_i+1}{2|\mathcal{N}(i)|}\cdot T_{i,j}\sum_{t'= 0}^{T_{i,j}-1}\norm{z_j^{t+1-t'} - z_j^{t-t'} }^2.
\end{align*}
By taking the telescope sum for $t=0, \ldots, T-1$, we have
\begin{align*}
&\quad\ \  \mathbb{E}_j[L(\mathbf{X}^{T}, \mathbf{Y}^{T}, \mathbf{z}^{T})] - \mathbb{E}_j[L(\mathbf{X}^0, \mathbf{Y}^0, \mathbf{z}^0)] \\
&\leq - \sum_{t=0}^{T-1}\sum_{(i, j) \in \mathcal{E}} \frac{\rho_i - 4L_{i,j}}{2|\mathcal{N}(i)|} (\norm{x_{i,j}^{t+1}-\tilde{z}_j^{t+1}}^2 + \norm{x_{i,j}^{t} - {z}_j^{t+1} }^2) - \sum_{t=0}^{T-1}\sum_{(i,j)\in\mathcal{E}}\frac{\gamma + \rho_{i}}{|\mathcal{N}(i)|} \cdot \norm{z_{j}^{t+1} - z_{j}^{t}}^2 \\
&\quad + \sum_{t=0}^{T-1}\sum_{(i,j)\in\mathcal{E}} \left(\frac{L_{i,j}^2(T_{i,j}+1)^2}{|\mathcal{N}(i)|} \left( \frac{1}{\rho_i} + \frac{1}{2} \right)  +  \frac{(4L_{i,j}+\rho_i+1)T_{i,j}^2}{2|\mathcal{N}(i)|} \right) \norm{z_j^{t+1} - z_j^{t}}^2 \\
&\leq - \sum_{t=0}^{T-1}\sum_{(i, j) \in \mathcal{E}}\beta_{i} (\norm{x_{i,j}^{t+1}-\tilde{z}_j^{t+1}}^2 + \norm{x_{i,j}^{t} - {z}_j^{t+1} }^2) - \sum_{t=0}^{T-1}\sum_{(i,j)\in\mathcal{E}} \alpha_{j} \norm{z_{j}^{t+1} - z_{j}^{t}}^2,
\end{align*}
where
\begin{align*}
	\alpha_j &:= (\gamma + \rho_i) - \sum_{i\in\mathcal{N}(j)}\left(\frac{1}{\rho_i} + \frac{1}{2} \right) L_{i,j}^2 (T_{i,j}+1)^2  - \sum_{i \in \mathcal{N}(j)} \frac{(4L_{i,j}+\rho_i+1)T_{i,j}^2}{2}, \\
	\beta_i &:= \frac{\rho_i - \max_{j \in \mathcal{N}(i)} 4L_{i,j}}{2|\mathcal{N}(i)|}.
\end{align*}
By making $\alpha_{j} > 0$ and $\beta_{i} > 0$ for all $(i,j) \in \mathcal{E}$, we prove the lemma.

\section{Proof of Lemma~\ref{lem:boundedbelow}}
\begin{proof}
From Lipschitz continuity assumption, we have.
\begin{equation*}
\begin{split}
f_i(\mathbf{z}^{t+1}) &\leq f_i(\mathbf{x}_i^{t+1}) + \sum_{j \in \mathcal{N}(i)} \innerprod{\nabla_j f_i(\mathbf{x}_i^{t+1}), z_j^{t+1} - x_{i,j}^{t+1}} 
+ \sum_{j \in \mathcal{N}(i)} \frac{L_{i,j}}{2} \norm{x_{i,j}^{t+1} - z_j^{t+1}}^2 \\
&= f_i(\mathbf{x}_i^{t+1}) + \sum_{j \in \mathcal{N}(i)}  \innerprod{\nabla_j f_i(\mathbf{x}_i^{t+1}) - \nabla_j f_i(\mathbf{z}^{t+1}), z_j^{t+1} - x_{i,j}^{t+1}} \\
&\quad +  \sum_{j \in \mathcal{N}(i)}  \innerprod{\nabla_j f_i(\mathbf{z}^{t+1}), z_j^{t+1} - x_{i,j}^{t+1}} + \sum_{j \in \mathcal{N}(i)} \frac{L_{i,j}}{2} \norm{x_{i,j}^{t+1} - z_j^{t+1}}^2 \\
&\leq f_i(\mathbf{x}_i^{t+1}) +  \sum_{j \in \mathcal{N}(i)} \innerprod{\nabla_j f_i(\mathbf{z}^{t+1}), z_j^{t+1} - x_{i,j}^{t+1}} + \sum_{j \in \mathcal{N}(i)} \frac{3L_{i,j}}{2} \norm{x_{i,j}^{t+1} - z_j^{t+1}}^2
\end{split}
\end{equation*}
Now we have
\begin{eqnarray*}
L(\mathbf{X}^{t+1}, \mathbf{z}^{t+1}, \mathbf{Y}^{t+1}) &=& h(\mathbf{z}^{t+1}) + \sum_{i=1}^N f_i(\mathbf{x}_i^{t+1}) + \sum_{j \in \mathcal{N}(i)} \innerprod{y_{i,j}^{t+1}, x_{i,j}^{t+1} - z_{j}^{t+1}} + \frac{\rho_{i}}{2} \norm{x_{i,j}^{t+1} - z_{j}^{t+1}}^2 \\
&\geq&h(\mathbf{z}^{t+1})+ \sum_{i=1}^N f_i(\mathbf{z}^{t+1}) +  \sum_{(i,j) \in \mathcal{E}} \innerprod{\nabla_j f_i(\tilde{\mathbf{z}}^{t+1}) - \nabla_j f_i(\mathbf{z}^{t+1}), z_j^{t+1} - x_{i,j}^{t+1}} \nonumber \\
&& + \sum_{(i,j) \in \mathcal{E}} \frac{\rho_{i}-3L_{i,j}}{2} \norm{x_{i,j}^{t+1} - z_j^{t+1}}^2 \\
&\geq& h(\mathbf{z}^{t+1})+ \sum_{i=1}^N f_i(\mathbf{z}^{t+1}) + \sum_{(i,j) \in \mathcal{E}} \frac{\rho_{i}-3L_{i,j}}{2} \norm{x_{i,j}^{t+1} - z_j^{t+1}}^2  \nonumber \\
&& -\sum_{(i,j) \in \mathcal{E}} L_{i,j}\norm{\tilde{\mathbf{z}}^{t+1} -\mathbf{z}^{t+1}} \norm{\mathbf{z}^{t+1} - \mathbf{x}_i^{t+1}} \\
&\geq& h(\mathbf{z}^{t+1})+ \sum_{i=1}^N f_i(\mathbf{z}^{t+1}) + \sum_{(i,j) \in \mathcal{E}} \left( \frac{\rho_{i}-4L_{i,j}}{2} \norm{x_{i,j}^{t+1} - z_j^{t+1}}^2 - \frac{L_{i,j}}{2}\norm{\tilde{\mathbf{z}}^{t+1} -\mathbf{z}^{t+1}}^2 \right)\\
&\geq& h(\mathbf{z}^{t+1})+ \sum_{i=1}^N f_i(\mathbf{z}^{t+1}) - \sum_{(i,j) \in \mathcal{E}} \frac{L_{i,j}}{2}\norm{\tilde{\mathbf{z}}^{t+1} -\mathbf{z}^{t+1}}^2 \nonumber \\
&\geq& \underline{f} - \mathrm{diam}^2(\mathcal{X}) \sum_{(i,j) \in \mathcal{E}} \frac{L_{i,j}}{2} > -\infty.
\end{eqnarray*}
\end{proof}

\subsection{Proof of Theorem~\ref{thm:convergence}}

\begin{proof}
From Lemma~\ref{lem:iteration}, we must have, as $t \to \infty$,
\begin{equation}
	x_{i,j}^{t+1} - \tilde{z}_{j}^{t+1} \to 0, \quad z_{j}^{t+1} - z_{j}^{t} \to 0, \quad x_{j}^{t} - z_{j}^{t+1} \to 0,\quad \forall (i, j) \in \mathcal{E}.
\end{equation}
Given Lemma~\ref{lem:ybound}, we have $y_{i,j}^{t+1} - y_{i,j}^{t} \to 0$. Since
\[
\norm{x_{i,j}^{t+1} - x_{i,j}^{t}} \leq \norm{x_{i,j}^{t+1} - \tilde{z}_{j}^{t+1}} + \norm{x_{j}^{t} - z_{j}^{t+1}} + \norm{\tilde{z}_{j}^{t+1} - z_{j}^{t+1}} \to 0,
\]
which proves \eqref{eq:kkt-3}, the first part.

For the second part, we have the following inequality from the optimality condition of \eqref{eq:update_z}:
\begin{align}
	0 \in \partial h_j(z_j^{t+1}) - \sum_{i\in \mathcal{N}(j)} \left( y_{i,j}^{t+1} + \rho_{i}(x_{i,j}^{t+1} - z_j^{t+1}) + \gamma (z_j^t - z_j^{t+1}) \right).
\end{align}
From \eqref{eq:lim_z} and \eqref{eq:kkt-3}, we have
\begin{equation}
	0 \in \partial h_j(z_j^{*}) - \sum_{i\in \mathcal{N}(j)} y_{i,j}^{*},
\end{equation}
which proves \eqref{eq:kkt-2}. Finally, from the optimality condition in \eqref{eq:update_z}, we have \eqref{eq:update_x} which implies \eqref{eq:kkt-1}, the second part of the theorem.

We now turn to prove the last part. Let $L'(\mathbf{X}, \mathbf{Y}, \mathbf{z}):= L(\mathbf{X}, \mathbf{Y}, \mathbf{z}) - h(\mathbf{z})$, which excludes $h(\mathbf{z})$ from $L(\mathbf{X}, \mathbf{Y}, \mathbf{z})$. Then, we have
\begin{align*}
	z_j - \nabla_{z_j} l'(\mathbf{X}, \mathbf{Y}, z_j) &= z_j - \sum_{i \in \mathcal{N}(j)} y_{i,j} - \sum_{i \in \mathcal{N}(j)} \rho_i (x_{i,j} - z_j)\\
	&= z_j - \sum_{i \in \mathcal{N}(j)} \rho_i (z_j - x_{i,j} - \frac{y_{i,j}}{\rho_i}).
\end{align*}
Therefore, we have
\begin{align}
	\norm{z_j^t - \prox{h}(z_j^t - \nabla_{z_j} l'(\mathbf{X}^t, \mathbf{Y}^t, z_j^t))}
	&\leq \norm{z_j^t - z_j^{t+1} + z_j^{t+1} - \prox{h}(z_j^t - \nabla_{z_j} l'(\mathbf{X}^t, \mathbf{Y}^t, z_j^t))} \nonumber \\
	&\leq \norm{z_j^t - z_j^{t+1}} + \norm{z_j^{t+1} - \prox{h}(z_j - \sum_{i \in \mathcal{N}(j)} \rho_i (z_j^t - x_{i,j}^t - \frac{y_{i,j}^t}{\rho_i}))} \nonumber  \\
	&\leq \norm{z_j^t - z_j^{t+1}} + \lVert \prox{h}(z_j^{t+1} - \sum_{i\in\mathcal{N}(j)}\rho_i(z_j^{t+1} - x_{i,j}^{t} - \frac{y_{i,j}^{t}}{\rho_i}) + \gamma(z_j^{t+1}-  z_j^{t})) \nonumber  \\
	&\quad\quad\quad\quad\quad\quad\quad\quad\quad - \prox{h}(z_j^t - \sum_{i \in \mathcal{N}(j)} \rho_i (z_j^t - x_{i,j}^t - \frac{y_{i,j}^t}{\rho_i})) \rVert \label{eq:thm-3-1}\\
	&\leq \left(2+\gamma+\sum_{i\in\mathcal{N}(j)}\rho_i \right)\norm{z_j^t - z_j^{t+1}}, \label{eq:thm-3-2}
\end{align}
where \eqref{eq:thm-3-1} is from the optimality in \eqref{eq:update_z} as
\begin{align*}
	z_j^{t+1} = \prox{h}(z_j^{t+1} - \sum_{i\in\mathcal{N}(j)}\rho_i(z_j^{t+1} - x_{i,j}^{t} - \frac{y_{i,j}^{t}}{\rho_i}) + \gamma(z_j^{t+1}-  z_j^{t})),
\end{align*}
and \eqref{eq:thm-3-2} is from the firm nonexpansiveness property of proximal operator. Then, by the update rule of $x_{i,j}^{t+1}$, if $x_{i,j}$ is selected to update at epoch $t$, we have
\begin{align*}
\norm{\nabla_{x_{i,j}} L(\mathbf{X}^t, \mathbf{Y}^t, \mathbf{z}^t)}^2 &= \norm{\nabla_j f_i(\mathbf{x}_i^t) + \rho_{i}(x_{i,j}^t - z_j^t + \frac{y_{i,j}^t}{\rho_{i}})}^2 \\
&= \norm{\nabla_j f_i(\mathbf{x}_i^t) - \nabla_j f_i(\tilde{\mathbf{z}}^{t})+(y_{i,j}^t - y_{i,j}^{t-1}) + \rho_{i}(\tilde{z}_{j}^{t} - z_j^t)}^2 \\
&\leq 3\norm{\nabla_j f_i(\mathbf{x}_i^t) - \nabla_j f_i(\tilde{\mathbf{z}}^{t})}^2 + 3\norm{y_{i,j}^t - y_{i,j}^{t-1}}^2 + 3\norm{\rho_{i}(\tilde{z}_{j}^{t} - z_j^t)}^2 \\
&\leq 3L_{i,j}^2\norm{x_{i,j}^t- \tilde{z}_{j}^{t}}^2 + 3\norm{y_{i,j}^t - y_{i,j}^{t-1}}^2 + 3\rho_{i}^2\norm{(\tilde{z}_{j}^{t} - z_j^t)}^2 \\
&\leq 3(L_{i,j}^2 + \rho_i^2)\norm{x_{i,j}^t- \tilde{z}_{j}^{t}}^2 + 3\rho_{i}^2\norm{(\tilde{z}_{j}^{t} - z_j^t)}^2,
\end{align*}
which implies that there must exist two positive constant $\sigma_1 > 0$ and $\sigma_2 > 0$ such that
\begin{equation}
	\sum_{(i,j)\in \mathcal{E}}\norm{\nabla_{x_{i,j}} L(\mathbf{X}^t, \mathbf{Y}^t, \mathbf{z}^t)}^2 \leq \sum_{(i,j)\in \mathcal{E}} \sigma_1 \norm{x_{i,j}^t- \tilde{z}_{j}^{t}}^2 +  \sum_{(i,j)\in \mathcal{E}} \sigma_2 \sum_{t'=0}^{T_{i,j}-1} \norm{z_j^{t-t'}-z_j^{t-t'-1}}.
	\label{eq:thm-3-3}
\end{equation}
The last step is to estimate $\norm{x_{i,j}^t - z_j^t}$, which can be done as follows:
\begin{align}
	\norm{x_{i,j}^t - z_{j}^t}^2 &\leq \norm{x_{i,j}^t - \tilde{z}_{j}^t}^2 + \norm{\tilde{z}_{j}^t - z_{j}^t}^2 \\
	\sum_{(i,j)\in \mathcal{E}} \norm{x_{i,j}^t - z_{j}^t}^2 &\leq \sum_{(i,j)\in \mathcal{E}}(\norm{x_{i,j}^t - \tilde{z}_{j}^t}^2 + \norm{\tilde{z}_{j}^t - z_{j}^t}^2 )
	\label{eq:thm-3-4}
\end{align}
Combining \eqref{eq:thm-3-2}, \eqref{eq:thm-3-3} and \eqref{eq:thm-3-4}, and summing up $t=0, \ldots, T-1$, we have
\begin{align}
	\sum_{t=0}^{T-1} P(\mathbf{X}^t, \mathbf{Y}^t, \mathbf{z}^t) &\leq \sum_{t=0}^{T-1} \sum_{(i,j)\in\mathcal{E}} \sigma_3 \norm{x_{i,j}^t - \tilde{z}_j^t}^2 + \sigma_4 T_{i,j}\norm{z_j^{t+1} - z_j^{t}}^2.
\end{align}
From Lemma~\ref{lem:iteration}, we have
\begin{align}
&\quad\quad L(\mathbf{X}^{T}, \mathbf{Y}^{T}, \mathbf{z}^{T}) - L(\mathbf{X}^0, \mathbf{Y}^0, \mathbf{z}^0) \\
&\leq - \sum_{t=0}^{T-1}\sum_{(i, j) \in \mathcal{E}}\beta_{i} (\norm{x_{i,j}^{t+1}-\tilde{z}_j^{t+1}}^2 + \norm{x_{i,j}^{t} - {z}_j^{t+1} }^2) - \sum_{t=0}^{T-1}\sum_{(i,j)\in\mathcal{E}} \alpha_{j} \norm{z_{j}^{t+1} - z_{j}^{t}}^2 \\
&\leq - \sum_{t=0}^{T-1}\sum_{(i, j) \in \mathcal{E}}\delta_1 \norm{x_{i,j}^{t+1}-\tilde{z}_j^{t+1}}^2 + \delta_2 \norm{z_{j}^{t+1} - z_{j}^{t}}^2,
\end{align}
where $\delta_1 := \min_{i} \beta_i$ and $\delta_2 := \min_{j} \alpha_j$.Now we can find some $C > 0$, such that the following equation hold:
\begin{align*}
	\sum_{t=0}^{T-1} P(\mathbf{X}^t, \mathbf{Y}^t, \mathbf{z}^t) &\leq C(L(\mathbf{X}^{0}, \mathbf{Y}^{0}, \mathbf{z}^{0}) - L(\mathbf{X}^T, \mathbf{Y}^T, \mathbf{z}^T)) \\
	&\leq C(L(\mathbf{X}^{0}, \mathbf{Y}^{0}, \mathbf{z}^{0}) - \underline{f}),
\end{align*}
where the last inequality we have used the fact that $L(\mathbf{X}^{t}, \mathbf{Y}^{t}, \mathbf{z}^{t})$ is lowered bounded by $\underline{f}$ for all $t$ from Lemma~\ref{lem:boundedbelow}. Let $T = T(\epsilon)$ and we have
\begin{equation}
	T(\epsilon) \leq \frac{C(L(\mathbf{X}^{0}, \mathbf{Y}^{0}, \mathbf{z}^{0}) - \underline{f})}{\epsilon},
\end{equation}
which proves the last part of Theorem~\ref{thm:convergence}.
\end{proof}

%% file: main.bbl
\begin{thebibliography}{}

\bibitem[\protect\citeauthoryear{Abadi \bgroup \em et al.\egroup
  }{2016}]{abadi2016tensorflow}
Mart{\'\i}n Abadi, Paul Barham, Jianmin Chen, Zhifeng Chen, Andy Davis, Jeffrey
  Dean, et~al.
\newblock Tensorflow: A system for large-scale machine learning.
\newblock In {\em 12th USENIX Symposium on Operating Systems Design and
  Implementation (OSDI 16)}, pages 265--283. USENIX Association, 2016.

\bibitem[\protect\citeauthoryear{Boyd \bgroup \em et al.\egroup
  }{2011}]{boyd2011distributed}
Stephen Boyd, Neal Parikh, Eric Chu, Borja Peleato, and Jonathan Eckstein.
\newblock Distributed optimization and statistical learning via the alternating
  direction method of multipliers.
\newblock {\em Foundations and Trends{\textregistered} in Machine Learning},
  3(1):1--122, 2011.

\bibitem[\protect\citeauthoryear{Chang \bgroup \em et al.\egroup
  }{2016a}]{chang2016asynchronous1}
Tsung-Hui Chang, Mingyi Hong, Wei-Cheng Liao, and Xiangfeng Wang.
\newblock Asynchronous distributed admm for large-scale optimization -- part i:
  Algorithm and convergence analysis.
\newblock {\em IEEE Transactions on Signal Processing}, 64(12):3118--3130,
  2016.

\bibitem[\protect\citeauthoryear{Chang \bgroup \em et al.\egroup
  }{2016b}]{chang2016asynchronous2}
Tsung-Hui Chang, Wei-Cheng Liao, Mingyi Hong, and Xiangfeng Wang.
\newblock Asynchronous distributed admm for large-scale optimization -- part
  ii: Linear convergence analysis and numerical performance.
\newblock {\em IEEE Transactions on Signal Processing}, 64(12):3131--3144,
  2016.

\bibitem[\protect\citeauthoryear{Chen \bgroup \em et al.\egroup
  }{2015}]{chen2015mxnet}
Tianqi Chen, Mu~Li, Yutian Li, Min Lin, Naiyan Wang, Minjie Wang, Tianjun Xiao,
  Bing Xu, Chiyuan Zhang, and Zheng Zhang.
\newblock Mxnet: A flexible and efficient machine learning library for
  heterogeneous distributed systems.
\newblock {\em arXiv preprint arXiv:1512.01274}, 2015.

\bibitem[\protect\citeauthoryear{Dean \bgroup \em et al.\egroup
  }{2012}]{dean2012large}
Jeffrey Dean, Greg Corrado, Rajat Monga, Kai Chen, Matthieu Devin, Mark Mao,
  Andrew Senior, Paul Tucker, Ke~Yang, Quoc~V Le, et~al.
\newblock Large scale distributed deep networks.
\newblock In {\em Advances in neural information processing systems}, pages
  1223--1231, 2012.

\bibitem[\protect\citeauthoryear{Friedman \bgroup \em et al.\egroup
  }{2001}]{friedman2001elements}
Jerome Friedman, Trevor Hastie, and Robert Tibshirani.
\newblock {\em The elements of statistical learning}, volume~1.
\newblock Springer series in statistics Springer, Berlin, 2001.

\bibitem[\protect\citeauthoryear{Hong \bgroup \em et al.\egroup
  }{2016a}]{hong2016convergence}
Mingyi Hong, Zhi-Quan Luo, and Meisam Razaviyayn.
\newblock Convergence analysis of alternating direction method of multipliers
  for a family of nonconvex problems.
\newblock {\em SIAM Journal on Optimization}, 26(1):337--364, 2016.

\bibitem[\protect\citeauthoryear{Hong \bgroup \em et al.\egroup
  }{2016b}]{hong2016unified}
Mingyi Hong, Meisam Razaviyayn, Zhi-Quan Luo, and Jong-Shi Pang.
\newblock A unified algorithmic framework for block-structured optimization
  involving big data: With applications in machine learning and signal
  processing.
\newblock {\em IEEE Signal Processing Magazine}, 33(1):57--77, 2016.

\bibitem[\protect\citeauthoryear{Hong}{2017}]{hong2017distributed}
Mingyi Hong.
\newblock A distributed, asynchronous and incremental algorithm for nonconvex
  optimization: An admm approach.
\newblock {\em IEEE Transactions on Control of Network Systems}, PP(99):1--1,
  2017.

\bibitem[\protect\citeauthoryear{Li \bgroup \em et al.\egroup
  }{2014a}]{li2014scaling}
Mu~Li, David~G Andersen, Jun~Woo Park, Alexander~J Smola, Amr Ahmed, Vanja
  Josifovski, James Long, Eugene~J Shekita, and Bor-Yiing Su.
\newblock Scaling distributed machine learning with the parameter server.
\newblock In {\em 11th USENIX Symposium on Operating Systems Design and
  Implementation (OSDI 14)}, pages 583--598, 2014.

\bibitem[\protect\citeauthoryear{Li \bgroup \em et al.\egroup
  }{2014b}]{li2014communication}
Mu~Li, David~G Andersen, Alexander~J Smola, and Kai Yu.
\newblock Communication efficient distributed machine learning with the
  parameter server.
\newblock In {\em Advances in Neural Information Processing Systems}, pages
  19--27, 2014.

\bibitem[\protect\citeauthoryear{Lian \bgroup \em et al.\egroup
  }{2015}]{lian2015asynchronous}
Xiangru Lian, Yijun Huang, Yuncheng Li, and Ji~Liu.
\newblock Asynchronous parallel stochastic gradient for nonconvex optimization.
\newblock In {\em Advances in Neural Information Processing Systems}, pages
  2737--2745, 2015.

\bibitem[\protect\citeauthoryear{Liu and Wright}{2015}]{liu2015asynchronous}
Ji~Liu and Stephen~J Wright.
\newblock Asynchronous stochastic coordinate descent: Parallelism and
  convergence properties.
\newblock {\em SIAM Journal on Optimization}, 25(1):351--376, 2015.

\bibitem[\protect\citeauthoryear{Liu \bgroup \em et al.\egroup
  }{2009}]{liu2009large}
Jun Liu, Jianhui Chen, and Jieping Ye.
\newblock Large-scale sparse logistic regression.
\newblock In {\em Proceedings of the 15th ACM SIGKDD international conference
  on Knowledge discovery and data mining}, pages 547--556. ACM, 2009.

\bibitem[\protect\citeauthoryear{Mota \bgroup \em et al.\egroup
  }{2013}]{mota2013d}
Jo{\~a}o~FC Mota, Jo{\~a}o~MF Xavier, Pedro~MQ Aguiar, and Markus Puschel.
\newblock D-admm: A communication-efficient distributed algorithm for separable
  optimization.
\newblock {\em IEEE Transactions on Signal Processing}, 61(10):2718--2723,
  2013.

\bibitem[\protect\citeauthoryear{Niu \bgroup \em et al.\egroup
  }{2011}]{recht2011hogwild}
Feng Niu, Benjamin Recht, Christopher Re, and Stephen Wright.
\newblock Hogwild: A lock-free approach to parallelizing stochastic gradient
  descent.
\newblock In {\em Advances in Neural Information Processing Systems}, pages
  693--701, 2011.

\bibitem[\protect\citeauthoryear{Taylor \bgroup \em et al.\egroup
  }{2016}]{taylor2016training}
Gavin Taylor, Ryan Burmeister, Zheng Xu, Bharat Singh, Ankit Patel, and Tom
  Goldstein.
\newblock Training neural networks without gradients: A scalable admm approach.
\newblock In {\em International Conference on Machine Learning}, pages
  2722--2731, 2016.

\bibitem[\protect\citeauthoryear{Tibshirani \bgroup \em et al.\egroup
  }{2005}]{tibshirani2005sparsity}
Robert Tibshirani, Michael Saunders, Saharon Rosset, Ji~Zhu, and Keith Knight.
\newblock Sparsity and smoothness via the fused lasso.
\newblock {\em Journal of the Royal Statistical Society: Series B (Statistical
  Methodology)}, 67(1):91--108, 2005.

\bibitem[\protect\citeauthoryear{Wei and Ozdaglar}{2013}]{wei20131}
Ermin Wei and Asuman Ozdaglar.
\newblock On the o (1/k) convergence of asynchronous distributed alternating
  direction method of multipliers.
\newblock {\em arXiv preprint arXiv:1307.8254}, 2013.

\bibitem[\protect\citeauthoryear{Zaharia \bgroup \em et al.\egroup
  }{2010}]{zaharia2010spark}
Matei Zaharia, Mosharaf Chowdhury, Michael~J Franklin, Scott Shenker, and Ion
  Stoica.
\newblock Spark: cluster computing with working sets.
\newblock {\em HotCloud}, 10:10--10, 2010.

\bibitem[\protect\citeauthoryear{Zhang and Kwok}{2014}]{zhang2014asynchronous}
Ruiliang Zhang and James Kwok.
\newblock Asynchronous distributed admm for consensus optimization.
\newblock In {\em International Conference on Machine Learning}, pages
  1701--1709, 2014.

\end{thebibliography}
